\documentclass[11pt,onecolumn]{article}
\author{
Arpan Mukherjee  \qquad  Ali Tajer
\thanks{The authors are with the Electrical, Computer, and System Engineering Department, Rensselaer Polytechnic Institute.}
}

\usepackage[margin=1in]{geometry}% <--- 1 in margin
\usepackage{setspace}
\doublespace                      % <--- double space

\usepackage{ISG_style}
\usepackage{times}
\usepackage{caption}
\usepackage[T1]{fontenc}
\usepackage{url}
\usepackage{ifthen}
\usepackage{cite}
\usepackage[cmex10]{amsmath}
\usepackage{amsthm, amssymb, amsfonts,graphicx,dsfont,mathrsfs}
\usepackage[ruled,vlined,linesnumbered]{algorithm2e}
\SetAlFnt{\small}
\usepackage{color}
\usepackage{flushend}
\usepackage{hyperref}
\newtheorem{theorem}{Theorem}
\newtheorem{proposition}{Proposition}
\newtheorem{definition}{Definition}
\newtheorem{lemma}{Lemma}

\newcommand*\diff{\mathop{}\!\mathrm{d}}

\DeclareMathOperator*{\arginf}{arg\,inf}
\DeclareMathOperator*{\argmax}{arg\,max}
\DeclareMathOperator*{\argmin}{arg\,min}
\DeclareMathAlphabet\mathbfcal{OMS}{cmsy}{b}{n}

\title{\bf \large SPRT-based Efficient Best Arm Identification in Stochastic Bandits}

\date{}

\begin{document}
\allowdisplaybreaks
\maketitle

\begin{abstract}
	This paper investigates the best arm identification (BAI) problem in stochastic multi-armed bandits in the fixed confidence setting. The general class of the exponential family of bandits is considered. 
	{The existing algorithms for the exponential family of bandits face computational challenges. {To mitigate these challenges, the BAI problem is viewed and analyzed as a sequential composite hypothesis testing task, and a framework is proposed that adopts the likelihood ratio-based tests known to be effective for sequential testing. Based on this test statistic, a BAI algorithm is designed that leverages the canonical sequential probability ratio tests for arm selection and is amenable to tractable analysis for the exponential family of bandits.} This algorithm has two key features: (1) its sample complexity is asymptotically optimal, and (2) it is guaranteed to be $\delta-$PAC. {Existing efficient approaches focus on the Gaussian setting and require Thompson sampling for the arm deemed the best and the challenger arm.} Additionally, this paper analytically quantifies the computational expense of identifying the challenger in an existing approach.} Finally, numerical experiments are provided to support the analysis.
	
\end{abstract}

\section{Introduction}

This paper considers the problem of best arm identification (BAI) in stochastic multi-armed bandits (MABs). In a stochastic MAB, each arm generates rewards from distributions with \emph{unknown} mean values. The objective of a learner in BAI is to identify the arm with the largest mean value, using the fewest number of samples.

The BAI problem is broadly studied under two key settings: the \emph{fixed budget} setting and the \emph{fixed confidence} setting. The objective in the fixed budget setting is to identify the arm having the largest mean within a pre-specified sampling budget while minimizing the decision error probability. On the other hand, in the fixed confidence setting, the learner identifies the best arm while having a guarantee on the error probability, and the objective is to minimize the sample complexity. The problem of identifying the best arm was first proposed in~\cite{Bubeck}, where the problem was posed in the fixed budget setting. More investigations in this setting can be found in~\cite{pmlr-v33-hoffman14} and~\cite{FB_JKS}. Representative studies in the fixed confidence setting can be found in~\cite{Gabillon,Kalyanakrishnan2012,pmlr-v49-garivier16a, LinGapE,Jamieson2014,pmlr-v28-karnin13}. Algorithms in this setting can be further classified into two categories, namely, non-Bayesian algorithms and Bayesian algorithms. 

{In the non-Bayesian setting, an optimal algorithm for parametric stochastic bandits was first proposed in~\cite{pmlr-v49-garivier16a} for the single parameter exponential family, which is based on a tracking procedure for arm selection. Specifically, in each round,~\cite{pmlr-v49-garivier16a} computes the optimal sampling proportion at the current mean estimates and selects an arm based on the optimal sampling proportion. While the track and stop (TaS) algorithm exhibits asymptotic optimality, it is computationally expensive due to solving an optimization problem to compute the optimal allocation in each iteration. To reduce the computational complexity, \cite{jedra2020optimal} showed that tracking the optimal proportion in intervals with exponentially increasing gaps is sufficient. However, this study focuses only on linear bandits with Gaussian noise. To address the computational complexity of track and stop, another approach to solving BAI is the gamification approach~\cite{degenne2019,pmlr-v119-degenne20a}. In this approach, BAI is posed as an unknown two-player game in which the sampling rules are obtained from the iterative strategies of the two players, which converge to a saddle point. The state-of-the-art algorithm for BAI in the non-Bayesian parametric setting is the Frank-Wolfe sampling algorithm (FW)~\cite{FW}. In this approach, the sampling rule is obtained from a single iteration of the Frank-Wolfe algorithm, which involves solving a two-player zero-sum game in each iteration. Despite being computationally efficient compared to TaS, the FW algorithm involves solving a linear program in each round.}

{Some of the non-Bayesian approaches to solving BAI for the non-parametric class of stochastic MABs (such as the class of sub-Gaussian stochastic MABs or bounded variance stochastic MABs) include confidence interval-based approaches (see~\cite{Gabillon,LinGapE,Jameison2014a}), successive elimination-based approaches (see~\cite{JMLR:v7:evendar06a,Jameison2014a,RAGE,Soare,Tao,pmlr-v28-karnin13}) and tracking based approaches~\cite{pmlr-v117-agrawal20a}.} In the confidence interval-based approach, the learner computes the sample mean of each arm, and a confidence interval around these empirical estimates, within which the true mean exists with a high probability. The rationale behind this strategy is to gather more evidence until there is no overlap between the confidence intervals and the learner decides the best arm based on the empirical estimates. On the other hand, the successive elimination-based strategy involves eliminating the potentially suboptimal arms in each round and continuing sampling from all other arms until only one arm remains to be eliminated.

While non-Bayesian approaches have been investigated extensively, the Bayesian setting is far less investigated. The first Bayesian algorithm was investigated in~\cite{russo2016}, which introduced the \emph{top-two} design philosophy and designed three Bayesian algorithms based on that. The top-two algorithmic design mitigated the computational challenge encountered in TaS, and provided a computationally simple alternative to the arm selection strategy. Among them, a modification of the Thompson sampling algorithm~\cite{Thompson1, Thompson2, Thompson3}, called the top-two Thompson sampling (TTTS) has received more attention due to its simplicity and optimality properties. The sample complexity of TTTS, however, was not analyzed in~\cite{russo2016}. To address the sample complexity of the top-two algorithms in Gaussian bandits, an improved algorithm was devised in~\cite{TTEI}, where the sample complexity was shown to be asymptotically optimal up to a constant factor. The sample complexity of TTTS in the Gaussian setting was later analyzed in~\cite{pmlr-v108-shang20a}, and it was shown to be asymptotically optimal. However, the sample complexity analysis was based on a non-informative prior, reducing it to a non-Bayesian setting. Furthermore, despite its simplicity, TTTS faces a computational challenge in its sampling strategy, which becomes significant in the regime of diminishing error rates. {To circumvent this, a transportation cost-based sampler was adopted in designing an efficient Thompson sampling-based BAI algorithm called the top-two transportation cost (T3C)~\cite{pmlr-v108-shang20a}. However, the sample complexity of T3C is only analyzed in the Gaussian setting.}

This paper leverages a sequential hypothesis testing framework for formalizing and solving the BAI problem in the fixed confidence setting. The arm selection and stopping rules are, in spirit, similar to the sequential probability ratio test~\cite{Wald1945}.  {BAI has been viewed as a hypothesis testing problem in a wide body of literature starting from the investigation in~\cite{pmlr-v49-garivier16a}, and subsequently in~\cite{TTEI,degenne2019,pmlr-v108-shang20a,FW}. Despite that, the algorithms offered generally do not adopt the statistics known to be effective for sequential composite hypothesis testing. In contrast, in this investigation, the arm selection rules dynamically update generalized likelihood ratios that compare the relative likelihood of different arms for being among the best. We refer to this algorithm by the top-two sequential probability ratio test (TT-SPRT). The idea of using SPRT-based rules was first introduced in~\cite{C89}, the key advantage of which is being amenable to analysis in the broader class of the exponential family of bandits. Specifically, the analysis for Thompson sampling-based approaches relies on non-asymptotic concentration inequalities for the convergence of the posterior mean to the ground truth. These concentration results exist in the literature for Gaussian bandits~\cite{pmlr-v108-shang20a}. However, for the broader class of the exponential family, the analysis of posterior sampling-based approaches is contingent on tail bounds for the posterior means, which need to be derived. Furthermore, empirically, computing a posterior distribution may involve Monte Carlo integration~\cite{russo2016}, in case a closed-form conjugate prior does not exist. Both of these issues make the log-likelihood ratio-based test statistic a more reasonable choice compared to posterior sampling.} Prior to this, the sample complexity of top-two algorithms~\cite{russo2016,TTEI,pmlr-v108-shang20a} has only been analyzed in the special setting of Gaussian bandits. This investigation is a generalization of~\cite{TTEI,pmlr-v108-shang20a} in the sense that the  TT-SPRT  algorithm is shown to be asymptotically optimal for the single-parameter exponential family. {Thus, we have addressed the open question of developing an {\em efficient} BAI algorithm in the fixed-confidence setting for the single parameter exponential family of distributions.} 

We show that in the special case of Gaussian bandits, TT-SPRT addresses a computational weakness of the TTTS algorithm.	Specifically, for dynamically identifying the top two arms, the TTTS sampling strategy generates random samples from the posterior distributions of arm rewards. Subsequently, the coordinate with the largest value in the first sample is deemed the best arm's index. For identifying the second arm (the challenger), TTTS keeps generating more samples until the coordinate with the largest value is distinct from the index already identified as the best arm. After enough explorations, the posterior distribution converges to the true model, and the largest coordinate of any random sample will be pointing to the best arm. This increases the delay in encountering a challenger. TT-SPRT does not have such a computational challenge. Finally, we note that our arm selection strategy is, in spirit, similar to the sequential probability ratio test\cite{Wald1945}~and~\cite{rev4}, which is a powerful test in a wide range of sequential testing problems, owing to its optimality properties and computational simplicity. 

{Besides this paper (and its earlier version~\cite{C89}) leveraging SPRT-based rules for the top-two algorithms with various choices of the top arm and the challenger arm have also been recently further investigated in~\cite{jourdan2022}, including the choice that has been considered in this paper. While the results in~\cite{jourdan2022} confirm the findings in this paper, there are a few critical differences in the assumptions on the bandit model considered. Specifically, for the exponential family analysis,~\cite{jourdan2022} assumes sub-Gaussian distributions and having distinguishable arm means. However, we make no such assumption for our analysis of the exponential family. In contrast, to account for this assumption, our algorithm involves a forced exploration stage.}

\section{BAI in Stochastic Bandits}

Consider a $K$-armed \emph{stochastic} MAB setting with $K$ arms generating rewards based on probability distributions that belong to an exponential family of distributions. Specifically, corresponding to a convex, twice-differentiable function $b : \R \mapsto \R$, we consider a single-parameter exponential family, denoted by 
%the class of measures 
\begin{align}
    \mathscr{P}_b\triangleq \left\{\nu_\theta \;:\;  \frac{\diff \nu_{\theta}}{\diff \nu} = \exp\left (\theta x - b(\theta) \right ) \; , \;  \theta\in\Theta\right\}\ ,
\end{align}
where $\Theta\subseteq \R$ is a compact parameter space, $\nu_\theta$ is the probability measure associated with parameter $\theta\in\Theta$, 
and $\nu$ is a reference measure on $\R$ such that $\nu_{\theta}\ll\nu$ for all $\theta\in\Theta$. The mean value of the distribution $\nu_\theta\in\mathscr{P}_b$ is given by $\mu(\theta)\triangleq \dot b (\theta)$, and the members of $\mathscr{P}_b$ can be uniquely identified by their mean values. We denote the \emph{unknown} mean of arm $i$ by $\mu_i$ and denote its associated probability distribution (probability density function for continuous and probability mass function for discrete distributions) from $\mathscr{P}_b$ by~$\nu(\cdot \med \mu_i)$. Accordingly, the vector of mean values is denoted by $\bmu\triangleq [\mu_1,\dots, \mu_K]$. The exponential family bandit model associated  with $\bmu$ is denoted by $\nu_{\bmu}\triangleq [\nu(\cdot \med \mu_1),\dots, \nu(\cdot \med \mu_K)]$. 
The \emph{best} arm, assumed to be unique, has the largest mean value and it is denoted by
\begin{align}
   a^\star\triangleq \argmax\limits_{i\in[K]}\;\mu_i\ .
\end{align}
The gap between the expected values of the best arm and  arm  $i\in[K]$ is denoted by $\Delta_i\triangleq \mu_{a^\star} - \mu_i$, 
%\begin{align}
%    \Delta_i\triangleq \mu_{a^\star} - \mu_i\ ,
%\end{align}
and the smallest gap among all possible arm pairs is captured by
\begin{align}
    \Delta_{\min}\triangleq \min_{i\neq j}|\mu_i - \mu_j|\ .
\end{align}
{We allow $\Delta_{\min}=0$ for the analysis of the single-parameter exponential family. In the special case of Gaussian bandits, we prove an implicit exploration property of our proposed algorithm with the additional assumption that $\Delta_{\min}>0$. However, such an assumption is not required in general, facilitated by the forced exploration stage in our algorithm design presented in Section~\ref{sec:algorithm}.} The learner performs a sequence of arm selections with the objective of identifying the best arm, i.e., $a^\star$, with the fewest number of samples (on average). In the sequential arm selection process, at time $n\in\N$ the leaner selects arm $A_n$ and receives the reward $X_n$, generating the filtration $\mcF_n\triangleq\{(A_t,X_t): t\in\{1,\dots, n\} \}$. Arm selection decision $A_{n+1}$ is assumed to be $\mcF_{n}-$measurable. Accordingly, the sequence of arm selections and the corresponding rewards obtained by the learner up to time $n$ are denoted by
\begin{align}
    \mcA_n\triangleq \{A_1,\cdots,A_n\}\ ,\;\;\text{and}\;\;\mcX_n\triangleq \{X_1,\cdots,X_n\}\ .
\end{align}
The sub-sequence of the rewards from only arm $i\in[K]$ up to time $n$ is denoted by
\begin{align}
    \mcX^i_n\triangleq \{X_s \;:\; s\in[n]\; , \; A_s = i\}\ .
\end{align}
We use the notation  $ D_{\sf KL} (\nu(\cdot\med\mu_i)\|\nu(\cdot\med\mu_j))$ to denote the Kullback-Leibler (KL) divergence from $\nu(\cdot\med\mu_j)$ to $\nu(\cdot\med\mu_i)$. It can be readily verified~\cite{pmlr-v49-garivier16a} that this divergence measure can be specified as a function of $\mu_i$ and $\mu_j$, for which we use the shorthand notation
\begin{align}
\label{eq:KL}
d_{\sf KL}(\mu_i\|\mu_j) & \triangleq D_{\sf KL} (\nu(\cdot\med\mu_i)\|\nu(\cdot\med\mu_j))  = b\left(\dot{b}^{-1}(\mu_j)\right) - b\left(\dot{b}^{-1}(\mu_i)\right) - \mu_i\left[\dot{b}^{-1}(\mu_j) - \dot{b}^{-1}(\mu_i)\right]\ .  
\end{align}
Let $\tau$, which is $\mcF_\tau$-measurable, denote the \emph{stochastic} stopping time of the BAI algorithm, that is the time instant at which the BAI algorithm terminates the search process and identifies an arm as the best arm.  Accordingly, define $\hat A_\tau$ as the arm identified as the best arm at the stopping time. In this paper, we consider the \emph{fixed confidence} setting, in which the learner's objective is to identify the best arm $a^\star$ with a pre-specified confidence level. We use $\delta-$PAC and $\beta$-optimality as the canonical notions of optimality for assessing the efficiency of the BAI algorithms. These notions are formalized next. First, we specify the $\delta-$PAC guarantee, which evaluates the fidelity of the terminal decision.
\begin{definition}[$\delta-$PAC]
A BAI algorithm is $\delta-$PAC 
%with confidence level $\delta\in(0,1)$ 
if the algorithm has a stopping time $\tau$ adapted to $\{\mcF_t:t\in\N\}$, and at the stopping time it ensures \begin{align}
    \P_{\bmu}\{\tau<+\infty,\;\hat A_\tau = a^\star\} > 1 - \delta\ ,
\end{align}
where $\P_{\bmu}$ denotes the probability measure induced by the interaction of the BAI algorithm with the bandit instance $\nu_{\bmu}$.
\end{definition}
\noindent For analyzing the sample complexity, we leverage a setting-specific notion of {\em problem complexity} defined next. The problem complexity characterizes the level of difficulty that an algorithm faces for identifying the best arm with sufficient fidelity. For this purpose, we define $\mathscr{W}(\beta)$ as the following set of $K$-dimensional probability simplexes 
\begin{align}\label{eq:W}
    \mathscr{W}(\beta)\triangleq \left\{\bomega\triangleq[\omega_1,\dots,\omega_K]\in[0,1]^K\; :\; \omega_{a^\star}=\beta,\;\sum\limits_{i\in[K]}\omega_i = 1 \right\}\ .
\end{align}

\begin{definition}[Problem Complexity] For a given bandit model $\nu_{\bmu}$, the problem complexity is defined as 
\begin{align}\label{eq:PC}
    \Gamma_{\bmu}(\beta) \triangleq \max\limits_{\bomega \in \mathscr{W}(\beta)}\min\limits_{i\neq a^\star}\; \left(\beta\; d_{\sf KL}(\mu_{a^\star}\|\mu_{a^\star,i}(\bomega)) + \omega_i\; d_{\sf KL}(\mu_i\|\mu_{a^\star,i}(\bomega))\right)\ ,
\end{align}
where we have defined
\begin{align}
    \mu_{i,j}(\bomega)\triangleq \frac{\mu_i\;\omega_i + \mu_j\;\omega_j}{\omega_i + \omega_j}\ .
\end{align}
Accordingly, we define the optimal sampling proportions as
\begin{align}\label{eq:omega}
    \bomega^\star(\beta) \triangleq \argmax\limits_{\bomega \in \mathscr{W}(\beta)}\min\limits_{i\neq a^\star}\; \left(\beta\; d_{\sf KL}(\mu_{a^\star}\|\mu_{a^\star,i}(\bomega)) + \omega_i\; d_{\sf KL}(\mu_i\|\mu_{a^\star,i}(\bomega))\right)\ .
\end{align}
\end{definition}
By leveraging the the notion of problem complexity, and defining $T_{n,i}$ as the number of times that arm $i\in[K]$ is pulled until $n$, we have the following known information-theoretic lower bound on the sample complexity of any $\delta-$PAC BAI algorithm. 
\begin{theorem}[Sample Complexity - Lower bound~\cite{pmlr-v108-shang20a}]\label{lemma:SC_LB}
For any $\delta-$PAC algorithm that satisfies
\begin{align}
\frac{T_{n,a^\star}}{n}\xrightarrow{n\rightarrow\infty}\beta\ ,    \end{align}
we almost surely have
\begin{align}
\label{theorem:LB}
    \liminf\limits_{\delta\rightarrow 0}\;\frac{\E_{\bmu}[\tau]}{\log(1/\delta)} \geq \frac{1}{\Gamma_{\bmu}(\beta)}\ .
\end{align}
\end{theorem}
The universal lower bound in Theorem~\ref{lemma:SC_LB} provides the minimum number of samples that any $\delta-$PAC BAI algorithm requires asymptotically, provided that $\beta$ fraction of measurement effort is allocated to the best arm. Accordingly, we specify $\beta$-optimality as a measure of sample complexity. 
\begin{definition}[$\beta$-optimality]
A $\delta-$PAC BAI algorithm is called asymptotically $\beta$-optimal if it satisfies
\begin{align}
    \frac{T_{n,a^\star}}{n}\xrightarrow{n\rightarrow\infty}\beta\quad\mbox{\rm almost surely}\ ,\quad \text{and}\;\;\limsup\limits_{\delta\rightarrow 0}\;\frac{\E_{\bmu}[\tau]}{\log(1/\delta)} \leq \frac{1}{\Gamma_{\bmu}}\ ,
\end{align}
where $\E_{\bmu}$ denotes expectation under the measure $\P_{\bmu}$.
\end{definition}

\section{TT-SPRT Algorithm for BAI}\label{sec:algorithm}

\noindent\textbf{A Hypothesis Testing Framework.} We pose the BAI problem as a collection of binary \emph{composite} hypothesis testing problems as follows. For each pair of distinct arms $(i,j)\in[K]^2$, define the following binary composite hypothesis, testing which discerns which of the two arms $i$ and $j$ has a larger mean value:
\begin{align}\label{eq:test}
    \mcH_{i,j} : \mu_i > \mu_j \ .
\end{align}
We highlight that unlike in sequential hypothesis testing that aims to solve problem~\eqref{eq:test}, our objective is \emph{not} forming decisions for all these ${K \choose 2}$ different hypotheses. Instead, we use them to form relevant statistical measures that capture the likelihood of different hypotheses. We then use these measures to design the arm selection and stopping rules. Specifically, we form generalized likelihood ratios as the key ingredient for our decision rules. To this end, at time $n$ and  corresponding to each arm pair $(i,j)$ we define the generalized log-likelihood ratio (GLLR) 
\begin{align}\label{eq:LLR}
    \Lambda_n(i,j) \triangleq \log\;  \frac{\sup_{\bmu}\P_{\bmu}(\mcX_n\;|\; \mcH_{i,j}  \mbox{ is true})}{\sup_{\bmu}\P_{\bmu}(\mcX_n\;|\; \mcH_{j,i} \mbox{ is true})}\ .
\end{align}
It can be readily verified that for the exponential family of bandits, $\Lambda_n(i,j)$ can be simplified to %\AT{should be fixed}
\begin{align}\label{eq:LLR1}
    \Lambda_n(i,j) = \log \; \frac{\max\limits_{\mu_i > \mu_j}\;  \prod\limits_{x_i\in\mcX_{n}^i}\nu (x_i\med \mu_i)\prod\limits_{x_j\in\mcX_{n}^j}\nu(x_j\med \mu_j)}{\max\limits_{\mu_j>\mu_i}\;\prod\limits_{x_i\in\mcX_{n}^i}\nu (x_i\med \mu_i)\prod\limits_{x_j\in\mcX_{n}^j}\nu(x_j\med \mu_j)}\ .
    %\nu(\mcX^i_n\med \mu_i)\;\nu(\mcX^j_n\med\mu_j)
\end{align}
% and in the Bernoulli bandit setting, $\Lambda_n(i,j)$ siplifies to
% \begin{align}\label{eq:LLR2}
%     \Lambda_n(i,j) \triangleq \log \frac{\max\limits_{\mu_i > \mu_j}\; p(\mcX^i_n\med \mu_i)\; p(\mcX^j_n\med\mu_j)}{\max\limits_{\mu_j>\mu_i}\;p(\mcX^i_n\med \mu_i)\; p(\mcX^j_n\med\mu_j)}\ .
% \end{align}
%These decisionsl the   the objective is to ultimately SPRT is known to be an optimal test for sequential binary hypothesis testing with i.i.d. data~\cite{SPRT_optimal}. At each time $n$, SPRT computes the log-likelihood ratio based on the accumulated data, and compares it with a lower and an upper threshold, to decide whether or not to collect more samples. Similarly, for every pair of arms $(i,j)\in[K]$, we compute the \emph{generalized} log-likelihood ratio, which is defined as
%At every instant, one of the arms should be more likely to have a larger mean compared to all other arms, which we would call the \emph{top arm}. First, let us define the generalized log-likelihood ratio between any pair of arms $i,j\in[K]$ at time $n$ as $\Lambda_n(i,j)$, i.e.,
%\begin{align}
%    \Lambda_n(i,j) \triangleq \log \frac{\max\limits_{\mu_i > \mu_j}\; f_i(\mcY_i^n\med \mu_i)f_j(\mcY_j^n\med\mu_j)}{\max\limits_{\mu_j>\mu_i}\;f_i(\mcY_i^n\med \mu_i)f_j(\mcY_j^n\med\mu_j)}\ ,
%\end{align}
It can be readily verified (for instance see~\cite{pmlr-v49-garivier16a}) that the GLLRs in~\eqref{eq:LLR1} have simple closed form, specified in the next lemma. To specify the closed form, we 
%define $T_{n,i}$ as the number of times that arm $i\in[K]$ is pulled until $n$, and 
define the empirical mean 
\begin{align}
\label{eq:sample_mean}
    \mu_{n,i}\triangleq \frac{1}{T_{n,i}}\sum\limits_{t\in[n]: a_t=i} X_t\ ,
\end{align}
as an empirical estimate of $\mu_i$. Accordingly, define the weighted average terms
\begin{align}
    \mu_{n,i,j}\triangleq \frac{T_{n,i}\mu_{n,i} + T_{n,j}\mu_{n,j}}{T_{n,i}+T_{n,j}}\ , \quad \forall i,j\in[K]\ .
\end{align}
\begin{lemma}[\cite{pmlr-v49-garivier16a}]\label{lemma:GLLR_closedform}
Under the exponential family of distributions, the GLLRs defined in~\eqref{eq:LLR1} have a closed forms given by
\begin{align}\label{eq:exp_LLR}
  \Lambda_n(i,j) & = \mathds{1}_{\{\mu_{n,i}>\mu_{n,j}\}}\cdot \Big[T_{n,i}\; d_{\sf KL}(\mu_{n,i} \| \mu_{n,i,j})  + T_{n,j}\;d_{\sf KL}(\mu_{n,j}\| \mu_{n,i,j}) \Big]  \ .
\end{align}
\end{lemma}
% and for Bernoulli distributions, the closed-form expression in (\ref{eq:exp_LLR}) simplifies to
% \begin{align}\label{eq:Bern_LLR}
% 	\Lambda_n(i,j)= &\Bigg[(T_{n,i}+T_{n,j})\log \mu_{n,i,j} - \left(T_{n,i}H(\mu_{n,i})+ T_{n,j}H(\mu_{n,j})\right)\Bigg]\cdot \mathds{1}_{\{\mu_{n,i}>\mu_{n,j}\}}\ ,
% \end{align}
% where $H$ represents the binary entropy function.
\noindent We note that in the Gaussian setting, the closed forms for the GLLRs in (\ref{eq:exp_LLR}) simplify to
\begin{align}\label{eq:gaussian_LLR}
    \Lambda_n(i,j)= \frac{1}{2\sigma^2}\cdot \frac{(\mu_{n,i}-\mu_{n,j})^2}{{T_{n,i}^{-1}}+{T_{n,j}^{-1}}}\cdot \mathds{1}_{\{\mu_{n,i}>\mu_{n,j}\}}\ .
\end{align}
Based on these definitions, the detailed steps of the TT-SPRT algorithm are specified in Algorithm~\ref{algorithm:TT_SPRT}. The decision rules involved are dicussed next. \vspace{.1 in}

%\AT{add the algorithm here}

%%%%%%%%% algorithm %%%%%%%%
		\begin{algorithm}[h]
%		\algsetup{linenosize=\small}
%		\setstretch{0.85}
		\caption{Top-Two SPRT (TT-SPRT)}
		\label{algorithm:TT_SPRT}
		
% 		\small
% 		\begin{algorithmic}
			\textbf{Input:} $\beta$
			
			\textbf{Initialize:} $n=0$, $\mcI_n = [K]$, $\mu_{n,i} = 0\;\forall\;i\in[K]$, $\Lambda(a_n^1,a_n^2)=0$, $c_{n,\delta} = 0$\\
			\While{$\Lambda(a_n^1,a_n^2) \leq c_{n,\delta}$}{
			    $n\leftarrow n + 1$\\
			    Sample $D_n\sim {\sf Bern}(\beta)$\\
			    Play an arm $a_{n}$ specified by~(\ref{eq:sampling rule_B}) or~(\ref{eq:sampling rule}) and obtain reward $X_n$\\
			    Update $\mu_{n,a_{n}}$ using~(\ref{eq:sample_mean})\\
			    $a^1_n \leftarrow \argmax\limits_{i\in[K]}\mu_{n,i}$\\
			    Compute $\Lambda_n(a_n,i)$ for every $i\in[K]\setminus\{a_n^1\}$ using~(\ref{eq:exp_LLR})\\
			    $a_n^2 \leftarrow \argmin\limits_{i\in[K]\setminus\{a_n^1\}} \Lambda_n(a_n^1,i)$\\
			    Update $c_{n,\delta}$ using~(\ref{eq:stop_exp}) or~(\ref{eq:stop_G}) \\
                Update $\mcI_n$ using~\eqref{eq:under-sampled}
			}
			\textbf{Output:} Top arm $a_n^1$
% 		\end{algorithmic}
	\end{algorithm}

\noindent \textbf{Arm Selection Rule.} At each instant $n$, the TT-SPRT identifies a \emph{top} arm as the the arm that has the largest sample mean $\mu_{n,i}$, denoted by
\begin{align}\label{eq:selectionrule1}
    a_n^1 \in  \argmax\limits_{i\in[K]} \; \mu_{n,i}\ .
\end{align}
% Based on \eqref{eq:exp_LLR}, that there is \emph{exactly one} arm that has a strictly positive log-likelihood ratio with respect to every other arm. Hence, the selection rule in~\eqref{eq:selectionrule1} simplifies to the following simple thresholding rule:
% \begin{align}
%     a_n^1 = i \quad \mbox{such that} \quad \forall\;j\in[K]\setminus\{i\}: \;\; \Lambda_n(i,j) > 0\ .
% \end{align}
In addition to the the top arm, we also define the \emph{challenger} arm as the one which is the closest competitor of the \emph{top} arm for being the {\em best} arm. The challenger arm at time $n$ is the arm that minimizes the log-likelihood ratio computed with respect to the top arm $a_n^1$. We denote the \emph{challenger} arm at time $n$ by $a_n^2$ and it is specified by
\begin{align}
    a_n^2\triangleq \argmin\limits_{j\in[K]: \mu_{n,j}<\mu_{n,a_n^1}} \Lambda_n(a_n^1,j)\ .
\end{align}
The TT-SPRT sampling strategy consists of an explicit exploration phase to ensure that each arm is pulled sufficiently often. Specifically, the arm selection strategy works as follows. We say an arm is \emph{under-explored}, if at time $n$ it is selected fewer than $\sqrt{n/K}$ times. Accordingly, define the set of under-explored arms at time $n$ as
\begin{align}
\label{eq:under-sampled}
    \mcI_n\triangleq\left\{i\in[K] : T_{n,i}\leq \left\lceil\sqrt{n/K}\right\rceil\right\}\ .
\end{align}
If $\mcI_n\neq \emptyset$, then TT-SPRT selects the arm in $\mcI_n$ that has been explored the least. Otherwise, it randomizes between the top arm $a_n^1$ and the challenger $a_n^2$ based on a Bernoulli random variable $D_n\sim{\sf Bern}(\beta)$, where $\beta\in(0,1)$ is a tunable parameter. This mechanism is included to ensure sufficient exploration of all arms. Hence, the arm action rule at time instant $n+1$ is given by
\begin{align}\label{eq:sampling rule_B}
	a_{n+1} \triangleq \left\{
	\begin{array}{ll}
	\argmin\limits_{i\in\mcI_n}\;T_{n,i}\ , & \mbox{if} \;\; \mcI_n\neq\emptyset\\
	a_n^1\ , & \mbox{if} \;\; D_n = 1\text{ and }\mcI_n = \emptyset\\
	a_n^2\ ,  & \mbox{if} \;\; D_n = 0\text{ and }\mcI_n = \emptyset\\
	\end{array}\right. \ .
\end{align}
We note that in the Gaussian bandit setting, the explicit form of the GLLR defined in~(\ref{eq:gaussian_LLR}) automatically ensures a sufficient exploration of the same order, i.e., at least $\sqrt{n/K}$ times for each arm. We will show this property in Section~\ref{sec:samplecomplexity} (Lemma~\ref{theorem:explore_gaussian}). However, this is not true in general, and the exploration mechanism is necessary. For example, in the case of Bernoulli bandits, unlike for the Gaussian bandits, the GLLRs for some of the arms, especially the ones with low average reward could be infinite in the early stages. Thus, TT-SPRT without the forced exploration phase might never sample these arms. Hence, in the special case of the Gaussian bandit setting, TT-SPRT does not require the explicit exploration phase, and is only based on randomizing between the top-two arms. This is formalized next. 

\vspace{.1 in}
\noindent\textbf{Arm Selection Rule for Gaussian Bandits.} Based on the choices of the top and challenger arms, at time~$n$, the TT-SPRT sampling strategy selects one of the two arms based on a Bernoulli random variable $D_n\sim{\sf Bern}(\beta)$, where $\beta\in(0,1)$ is a tunable parameter. Specifically, given $\mcF_n$-measurable decisions $a^1_n$ and $a^2_n$ at time $n$, the arm to be sampled at time $n+1$ is specified by the randomized rule
\begin{align}\label{eq:sampling rule}
	a_{n+1} \triangleq \left\{
	\begin{array}{ll}
	a_n^1\ , & \mbox{if} \;\; D_n = 1\\
	a_n^2\ ,  & \mbox{if} \;\; D_n = 0\\
	\end{array}\right. \ .
\end{align}
 
\noindent\textbf{Stopping Rule.} We specify the following thresholding-based stopping criterion, in which we select the threshold such that the algorithm ensures the $\delta-$PAC guarantee:\begin{align}\label{eq:stop}
    \tau\triangleq \inf\Big\{n\in\N : \Lambda_n(a_n^1,a_n^2) > c_{n,\delta}\Big\}\ .
\end{align}
This stopping rule at each time $n$ evaluates the GLLR between the top and challenger arms, i.e., $\Lambda_n(a_n^1,a_n^2)$, and terminates the sampling process as soon as this GLLR exceeds a pre-specified threshold $c_{n,\delta}$. This threshold will be specified in Section~\ref{sec:results}. %for the general setting of the exponential family, and it will be specialized to the Gaussian setting in Section~\ref{sec:results_G}. 
We note that the stopping criterion is different from that of the canonical SPRT's for sequential binary hypothesis testing, which involves two thresholds in order to control the two types of decision errors. In contrast to hypothesis testing, for BAI we have only one error event $\P_{\bmu}(\hat A_\tau\neq a^\star)$ to control.

\section{Main Results}
\label{sec:results}
% In this section, we establish the TT-SPRT algorithm's optimality for the general exponential family in Section~\ref{sec:deltapac} and Section~\ref{sec:samplecomplexity}, and also specialize them to the Gaussian setting. The results in the Gaussian setting have two main distinctions from the results for the general exponential family. First, we show that the exploration mechanism for the Gaussian setting is not necessary. Secondly, we provide a less stringent stopping rule, according to which the sampling process requires fewer samples. Both of these lead to improving the sample complexity by refraining from unnecessary exploration and terminating the process sooner. Based on these specialized rules, we recover the optimality guarantees that were also established in~\cite{TTEI,pmlr-v108-shang20a}. Even though the TT-SPRT achieves the optimality guarantees similar to those in~\cite{TTEI,pmlr-v108-shang20a}, it has a major computational advantage, which we discuss in Section~\ref{sec:TTTS}.
In this section, we establish the TT-SPRT algorithm's optimality for the general exponential family in Section~\ref{sec:performance exponential}. Subsequently, in Section~\ref{sec:samplecomplexity}, we specialize them to the Gaussian setting. The results in the Gaussian setting have two main distinctions from the results for the general exponential family. First, we show that the exploration mechanism for the Gaussian setting is not necessary. Secondly, we provide a less stringent stopping rule, according to which the sampling process requires fewer samples. 
%Both of these lead to improving the sample complexity by refraining from unnecessary exploration and terminating the process sooner. 
Based on these specialized rules, we recover the optimality guarantees that were also established in~\cite{TTEI,pmlr-v108-shang20a}. Even though the TT-SPRT achieves the optimality guarantees similar to those in~\cite{TTEI,pmlr-v108-shang20a}, it has a major computational advantage, which we discuss in Section~\ref{sec:TTTS}.

\subsection{Exponential Family of Bandits}
%\subsection{$\delta-$PAC Guarantee} 
\label{sec:performance exponential}
Our main technical results are on the optimality of the sample complexity. Specifically, we demonstrate that the sample complexity of the TT-SPRT algorithm coincides with the known information-theoretic lower bounds on the sample complexity of the algorithms that are $\delta-$PAC. Hence, as the first step, and for completeness, we provide the results that establish that TT-SPRT algorithm is $\delta-$PAC. An algorithm being $\delta-$PAC, generally, depends only on its design of the terminal decision rule for identifying the best arm and is independent of the arm selection strategy. For this purpose, we leverage an existing result for the exponential family~\cite{Kaufmann_JMLR}, which demonstrates that the combination of any arm selection strategy, the stopping rule specified in~(\ref{eq:stop}), and a proper choices of the threshold $c_{n,\delta}$ ensures a $\delta-$PAC guarantee. For completeness, the result is provided in the following theorem.
{\begin{theorem}[$\delta-$PAC -- Exponential]
\label{theorem:PAC_bern}
For the exponential family, the stopping rule in~(23) with the choice of the threshold
\begin{align}
\label{eq:stop_exp}
    c_{n,\delta}\triangleq 2\mcC_{\rm exp}\ \left( \log\frac{K-1}{\delta} \right ) + 6\log\left( \log\frac{n}{2} + 1\right)\ ,
\end{align}
coupled with any arm selection strategy is $\delta-$PAC, where the function $\mcC_{\rm exp} : \mathbb{R}_+\mapsto \mathbb{R}_+$ is defined similarly to~\cite{Kaufmann_JMLR}, and satisfies $\mcC_{\rm exp}(x) \simeq x + 4\log(1+x + \sqrt{2x})$.
\end{theorem}
\begin{proof}
Follows similarly to~\cite[Theorem 7]{Kaufmann_JMLR}.
\end{proof}}
To show that the TT-SPRT algorithm achieves the optimal sample complexity, we provide an upper bound on its sample complexity, which matches the information-theoretic lower bound on the sample complexity of any $\delta-$PAC BAI algorithm. Note that the explicit exploration phase incorporated in TT-SPRT ensures that each arm is explored sufficiently often and that the sample means converge to the true values. Before delineating the sample complexity achieved by the TT-SPRT, we provide the key properties of the TT-SPRT arm selection strategy. The first property pertains to the sufficiency of exploration. 
\begin{lemma}[Sufficient Exploration -- Exponential]
\label{theorem:explore_exp}
Under the TT-SPRT sampling strategy~(\ref{eq:sampling rule_B}), for all $i\in[K]$ and for any $n\in\N$, we have $T_{n,i}\geq \sqrt{n/K}-1$. 
\end{lemma}
\begin{proof}
{The proof follows similar arguments as in the proof of~\cite[Lemma 8]{pmlr-v49-garivier16a}.}
\end{proof}
\noindent This property ensures that each arm is sampled sufficiently often, such that the empirical estimates converge to the true mean value if the sampling strategy is allowed to continue drawing samples without stopping. The second property pertains to the convergence of the empirical mean estimates to the true mean values. Specifically, we show that there exists a time after which the mean empirical values are within an $\epsilon$-band of their corresponding ground truth values. 
\begin{lemma}[Convergence in Mean -- Exponential]
\label{theorem:mean:convergence:exp}
For any $\epsilon>0$, define $N_\epsilon^\mu$ as 
\begin{align}
    N_\epsilon^\mu \triangleq \inf\{s\in\N \;:\; |\mu_{n,i} - \mu_i | \leq \epsilon \;\; \forall n\geq s,\;\;\forall i\in[K]\}\ .
\end{align}
Under the sampling rule in~\eqref{eq:sampling rule_B} for the exponential family we have $\E_{\bmu}[N_\epsilon^\mu]<+\infty$.
\end{lemma}
\begin{proof}
See Appendix~\ref{proof:mean:convergence:exp}.
\end{proof}
{The convergence in mean values for the case of Gaussian bandits has been analyzed in~\cite{TTEI,pmlr-v108-shang20a} under different top-two sampling rules. The analysis mainly relies on a concentration inequality for the convergence of the sample mean to the ground truth. Furthermore, the analysis relies on the assumption that $\Delta_{\min}>0$, despite showing an implicit exploration of the arms due to the sampling rule. In contrast, our analysis for the exponential family makes no such assumption at the cost of an explicit exploration phase. Our analysis leverages the explicit exploration phase and Chernoff's inequality for the convergence in mean. We refer to Appendix~\ref{proof:mean:convergence:exp} for more details.} The third property is that the TT-SPRT ensures that the ratio of the sampling resources spent on arm $i$ up to time $n$, i.e., $T_{n,i}/n$ converges to the optimal sampling proportion $\omega^*_i(\beta)$ when $n$ is sufficiently large. This is formalized in the next theorem.
\begin{lemma}[Convergence to Optimal Proportions]
\label{theorem:proportions}
%\begin{enumerate}
    Under the sampling rule in~\eqref{eq:sampling rule_B} for the exponential family of bandits, there exists a stochastic time $N_\epsilon^\omega$, $\E_{\bmu}[N_\epsilon^\omega]<+\infty$, such that
    \begin{align}
        \left\lvert \frac{T_{n,i}}{n} - \omega^\star_i(\beta)\right\rvert \leq \epsilon\quad \forall n > N_\epsilon^\omega\ ,\;\forall i \in[K]\ .
    \end{align}
\end{lemma}
\begin{proof}
See Appendix~\ref{proof:exp_allocation}.
\end{proof}
\noindent {Convergence in the sampling proportions to the $\beta$-optimal allocation is a key step in the sample complexity analysis of top-two algorithms, as shown in~\cite{TTEI,pmlr-v108-shang20a}. The proof for the Convergence in allocation can be broken down into two steps: 1) convergence of the sampling proportion of the best arm to $\beta$, and 2) convergence of the sampling proportion of the remaining arms to the $\beta$-optimal allocation. The novelty in the analysis of TT-SPRT comes in the second part. Specifically, TT-SPRT uses the GLLR statistic for arm selection, which differs from the statistic in~\cite{TTEI,pmlr-v108-shang20a}. The analysis for Convergence in sampling proportion relies on showing that, eventually, the challenger is always contained in the set of the under-sampled arms. For details, we refer to Appendix~\ref{proof:exp_allocation}.} Finally, we use the notion of $\beta$-optimality, which was first introduced in establishing optimality guarantees for the top-two algorithms for BAI in~\cite{TTEI}, and was adopted later  for establishing the optimality of the TTTS algorithm~\cite{pmlr-v108-shang20a}. Specifically, we show that the TT-SPRT algorithm is $\beta$-optimal, i.e., its sample complexity asymptotically matches the universal lower bound on the average sample complexity provided in Theorem~\ref{lemma:SC_LB}, while satisfying the condition on the sampling proportion allocated to the best arm $a^\star$. By leveraging the three properties shown in Lemma~\ref{theorem:explore_exp}-Lemma~\ref{theorem:proportions}, we provide an upper bound on the average sample complexity of the TT-SPRT algorithm. 
\begin{theorem}[Sample Complexity - Upper bound]
\label{theorem:SC_exp}
The TT-SPRT algorithm is $\delta-$PAC and satisfies
\begin{align}
    \frac{T_{n,a^\star}}{n}\xrightarrow{n\rightarrow\infty}\beta\ .
\end{align}
Furthermore, for the exponential family of distributions, we have 
\begin{align}
\label{eq:SC_bern}\lim\limits_{\delta\rightarrow 0}\;\frac{\E_{\bmu}[\tau]}{\log(1/\delta)} \leq \frac{1}{\Gamma_{\bmu}(\beta)}\ ,
\end{align}
\end{theorem}
\begin{proof}
See Appendix~\ref{proof:SC_exp}.
\end{proof}
\noindent Theorem~\ref{theorem:SC_exp} along with the universal lower bound in Theorem~\ref{lemma:SC_LB} establishes the asymptotic $\beta$-optimality of TT-SPRT for the exponential family of bandits.

\subsection{Gaussian Bandits}
\label{sec:samplecomplexity}
{Next, we specialize the properties of the TT-SPRT algorithm to the Gaussian bandit setting.} First, we leverage a tighter stopping condition from~\cite{pmlr-v108-shang20a}, which ensures that in the specific setting of Gaussian bandits, the stopping rule specified in~(\ref{eq:stop}) is $\delta-$PAC. {Specifically, the threshold for the exponential family scales as $O(\log(\log n)^6)$, while the one for the Gaussian setting scales as $O(\log(\log n)^4)$.} 
%We start by proving that the TT-SPRT algorithm is a $\delta-$PAC algorithm. For this purpose, we leverage an existing result in~\cite{pmlr-v108-shang20a} which demonstrates that the combination of any arm selection strategy, the stopping rule specified in~(\ref{eq:stop}), and a proper choice of the threshold $c_{n,\delta}$ ensures $\delta-$PAC guarantee. With appropriate adjustments, this result ensures $\delta-$PAC guarantee for the TT-SPRT algorithm as well. For completeness, this result is presented in the following theorem. 
\begin{theorem}[$\delta-$PAC -- Gaussian]
In the Gaussian bandit setting, the stopping rule in~(\ref{eq:stop}) with the choice of the threshold
\begin{align}
\label{eq:stop_G}
    c_{n,\delta}\triangleq 4\log(4+\log n) + 2g\left(\frac{\log(K-1)-\log\delta)}{2}\right )\ ,
\end{align}
where we have defined $g(x) \triangleq  x + \log x$, coupled with any arm selection strategy is $\delta-$PAC.
\end{theorem}
\begin{proof}
Follows similarly to \cite[Theorem 2]{pmlr-v108-shang20a}.
\end{proof}
Next, we show that the explicit exploration phase is redundant in the Gaussian setting, since the TT-SPRT sampling rule automatically ensures sufficient exploration in this setting. 
\begin{lemma}[Sufficient Exploration -- Gaussian]\label{theorem:explore_gaussian}
The TT-SPRT sampling strategy~(\ref{eq:sampling rule}) ensures that there exists a random variable $N_1$ such that $\E_{\bmu}[N_1]<+\infty$ and for all $i\in[K]$ and for any $n>N_1$, we have $T_{n,i}\geq\sqrt{n/K}$ almost surely. 
%\AT{the lower bounds and the ``almost surely'' part are inconsistent with the exponential result. Are they both correct as presented?}.
\end{lemma}
\begin{proof}
See Appendix~\ref{proof:exploration_Gaussian}.
\end{proof} 
\noindent{Similarly to Lemma~\ref{theorem:mean:convergence:exp}, we show that due to the implicit exploration of the TT-SPRT sampling rule for Gaussian bandits, we have convergence in the empirical mean values to the ground truth.}
\begin{lemma}[Convergence in Mean -- Gaussian]
\label{theorem:mean:convergence:gaussian}
Under the sampling rule in~\eqref{eq:sampling rule} for the Gaussian setting, there exists a stochastic time $M_\epsilon^\mu$, $\E_{\bmu}[M_\epsilon^\mu]<+\infty$, such that for all $n>M_\epsilon^\mu$, we have:
    \begin{align}
        |\mu_{n,i}-\mu_i|\;<\;\epsilon\quad{\rm almost\;surely}\ .
    \end{align}
\end{lemma}
\begin{proof}
See Appendix~\ref{proof:mean:convergence:exp}.
\end{proof}

\noindent {We show that the TT-SPRT algorithm in the Gaussian setting is $\beta$-optimal by providing an upper bound on the average sample complexity of the TT-SPRT algorithm in the Gaussian setting. }
\begin{theorem}[Sample Complexity - Upper bound]
\label{theorem:SC_Gaussian}
The TT-SPRT algorithm is $\delta-$PAC and satisfies
\begin{align}
    \frac{T_{n,a^\star}}{n}\xrightarrow{n\rightarrow\infty}\beta\ .
\end{align}
Furthermore, for the Gaussian model, we have
\begin{align}
\label{eq:SC_G}\lim\limits_{\delta\rightarrow 0}\;\frac{\E_{\bmu}[\tau]}{\log(1/\delta)} \leq \frac{1}{\Gamma_{\bmu}(\beta)}\ .
\end{align}
\end{theorem}
\begin{proof}
See Appendix~\ref{proof:SC_exp}.
\end{proof}

\subsection{Challenger Identification in Top-Two Thompson Sampling} \label{sec:TTTS}
The TTTS algorithm~\cite{russo2016}, is a Bayesian algorithm in which the reward mean values are assumed to have the prior distribution $\mcN(0,\kappa^2)$. Based on this prior, at each time $n$ and based on $\mcX_n$, the learner computes a posterior distribution $\Pi_n\in\R^K\to\R$. Specifically, for the average reward realization $\bar\bmu$ and reward realization $\mcX_n = \bx_n$:
\begin{align}\label{eq:Pi}
    \Pi_n(\bar\bmu\;|\; \bx_n) \triangleq \displaystyle\prod\limits_{i=1}^K\frac{1}{\sqrt{2\pi\eta_{n,i}^2}}\exp\left\{-\frac{(\bar\mu_i - \tilde\mu_{n,i})^2}{2\eta_{n,i}^2}\right\}\ ,
\end{align}
%     \P(\tilde\bmu = \bmu \;|\; \mcY^n)\ .
%{Should it be a ``parameter realization'', instead of a ``reward realization''? The posterior is on the parameter space, conditional on previous rewards. Sorry if I am missing something trivial} \AT{You're right. Fixed it. Double check.}
%The marginal posterior reward distribution of arm $i\in[K]$ is Gaussian with mean $\tilde\mu_{n,i}$ and variance $\eta^2_{n,i}$ given by
where we have defined
\begin{align}\label{eq:TTTS_mean}
    \tilde\mu_{n,i} & \triangleq \frac{1}{T_{n,i}+\sigma^2/\kappa^2}\sum\limits_{x\in\mcX^i_n}x\ , \qquad \mbox{and} \qquad 
    \eta_{n,i}^2  \triangleq \frac{\sigma^2}{T_{n,i}+\sigma^2/\kappa^2}\ .
\end{align}
While the TTTS is devised for the setting with a Gaussian prior distribution for the rewards, the sample complexity analysis for the algorithm holds  
%for \emph{improper} priors~\cite{pmlr-v108-shang20a}. Specifically, it is assumed that 
for the asymptotic regime of $\kappa \rightarrow +\infty$. This assumption renders the prior distributions uninformative, and the setting becomes equivalent to that of the non-Bayesian counterpart, i.e., when the means corresponding to each arm is unknown, and we have no prior distribution over the arm means. Thus, the posterior mean corresponding to each arm $i\in[K]$ defined in~(\ref{eq:TTTS_mean}) reduces to that of the sample mean, i.e., $\tilde\mu_{n,i}=\mu_{n,i}$, and the settings for both TTTS as well as TT-SPRT become equivalent. As a result,  $\Pi_n$ denotes the product of the $K$ Gaussian posteriors, $\mcN(\mu_{n,i}, \sigma^2_{n,i})$ for all $i\in[K]$, where we have defined $\sigma_{n,i}^2\triangleq \sigma^2/T_{n,i}$. Let us denote the expectation operator with respect to $\Pi_n$ by $\E_n$.

The arm selection strategy of the TTTS algorithm works as follows. At each time $n$, a random $K$-dimensional sample $\btheta^n\triangleq (\theta^n_1,\cdots \theta^n_K)$ is drawn from the posterior distribution $\Pi_n$. The coordinate with the largest value is defined as the index of the top arm, denoted by $b^1_n\triangleq\argmax_{i\in[K]} \theta^n_i$. In order to find a challenger (the closest competitor to $b^1_n$), the algorithm continues sampling the posterior $\Pi_n$ until a realization from $\Pi_n$ is encountered such that the index of its largest coordinate is distinct from $b^1_n$. This is considered the challenger arm and its index is denoted by $b^2_n$. Encountering a challenger arm requires generating enough samples from $\Pi_n$. As $n$ increases and the posterior distribution $\Pi_n$ points to more confidence about the best arm, the number of samples required to encounter a challenger increases. We denote a sample $s$ generated from $\Pi_n$ by $\btheta^n_s\triangleq (\theta_{s,1}^n,\cdots,\theta_{s,K}^n)$. By design, clearly, $b^2_n\triangleq\argmax_{i\in[K]}\theta^n_{s,i}$, and $b^2_n\neq b^1_n$. Once $b^1_n$ and $b^2_n$ are identified, the TTTS selects one of them based on a Bernoulli random variable parameterized by $\beta\in(0,1)$.  As mentioned earlier, as $n$ increases and $\Pi_n$ converges, the number of samples required for encountering a challenger also increases, and this imposes a computational challenge, especially for large $n$. In the next theorem, we show that the number of samples required for encountering a challenger scales at least exponentially in $\sqrt{n}$. For this purpose, we define
\begin{align}
    T_{{\sf TTTS}}^n\triangleq\inf\{s\in\N \;:\; \exists i\in[K]\; , \; \theta_{s,i}^n>\theta_{s,b^1_n}^n\}\ ,
\end{align}
as the number of posterior samples required for finding a challenger at time $n$.
\begin{theorem}[Challenger's Sample Complexity]\label{theorem:TTTS_samples}
In the TTTS algorithm~\cite{russo2016}, there exists a random variable $N_0$ such that for any $n>N_0$, the average number of posterior samples required in order to find a challenger is lower-bounded as
\begin{align}\label{eq:post_LB}
&\E_n[T_{{\sf TTTS}}^n]\geq\min\limits_{i\in[K]\setminus\{a^\star\}}\;2\exp\left(\sqrt{\frac{n}{K}}\; C_{i,{\sf L}}\right)\ .
\end{align}
Furthermore, for all $n>\max\{N_0,32\sigma^2/9\Delta_{\min}^2\}$, we have
\begin{align}\label{eq:post_UB}
    \E_n[T^n_{\sf TTTS}]\leq&\;\;\max\limits_{i\in[K]\setminus\{a^\star\}} \sqrt{2\pi \e}\exp\left( n\;C_{i,{\sf U}}\right)\ ,
\end{align}
where we have defined
\begin{align*}
    C_{i,{\sf L}} \triangleq  \frac{(\Delta_i - \Delta_{\min}/2)^2}{4\sigma^2}\;\;  \mbox{and} \;\; C_{i,{\sf U}} \triangleq \frac{(\Delta_i+\Delta_{\min}/2)^2}{2\sigma^2}\ .
\end{align*}
%\begin{align}
%&\E_{\bmu}[T_{{\sf TTTS}}^n]\geq\min\limits_{i\in[K]\setminus\{a^\star\}}\;2\exp\Bigg\{\sqrt{\frac{n}{K}}\frac{(\Delta_i - \Delta_{\min}/2)^2}{4\sigma^2}\Bigg\}\ ,
%\end{align*}
%where $\Delta_{\min}\triangleq \min_{i,j\in[K]}|\mu_i - \mu_j|$. Furthermore, for all $n>\max(N_0,32\sigma^2/9\Delta_{\min}^2)$, we have
%\begin{align}
%    \E_{\bmu}[T^n_{\sf TTTS}]\leq&\;\;\min\limits_{i\in[K]\setminus\{a^\star\}} \sqrt{2\pi e}\exp\Bigg\{ n\frac{(\Delta_i+\Delta_{\min}/2)^2}{2\sigma^2}\Bigg\}\ .
%\end{align}
\end{theorem}
\begin{proof}
%See supplementary document~\cite{supplemental}.%Appendix~\ref{appendix:B}.
See Appendix~\ref{appendix:D}. 
\end{proof}

We observe that the lower bound increases exponentially in $\sqrt{n}$, and thus, diverges for large values of $n$, i.e., when the confidence required on the decision quality is large. Furthermore, note that for a sufficiently small $\delta$, the TTTS stopping time always exceeds $N_0$. Specifically, we show in Appendix~\ref{appendix:D} that there exists $\delta(N_0)>0$ such that for any $\delta\in(0,\delta(N_0)$, the TTTS algorithm almost surely stops after $N_0$ time instants. This implies that for a large enough decision confidence, the TTTS algorithm requires a large number of posterior samples to identify a challenger, which has an exponential growth of the order of at least $\sqrt{n}$.

\section{Numerical Experiments}

In this section, we conduct numerical experiments to compare the performance of TT-SPRT against state-of-the-art algorithms. {Specifically, we provide experiments for empirically depicting the computational difficulty of the TTTS algorithm in identifying a challenger arm. Furthermore, we compare the performance of TT-SPRT against existing BAI strategies under three different bandit models, namely, Gaussian bandits, Bernoulli bandits, and exponential bandits. {While $\beta$ is a tunable parameter of TT-SPRT, and its optimal choice depend on instance-specific parameters, a choice of $\beta=0.5$ exhibits good empirical performance over several bandit instances. Specifically, the average sample complexity with $\beta=0.5$ is at most twice the asymptotically optimal sample complexity, shown in~\cite[Lemma 3]{russo2016}. Unless otherwise stated, we have used $\beta=0.5$ for simulations involving the top-two sampling strategies.}}

\begin{figure*}[t]
	\centering
	\begin{minipage}{0.49\textwidth}
		\centering
		\includegraphics[width=0.99\textwidth]{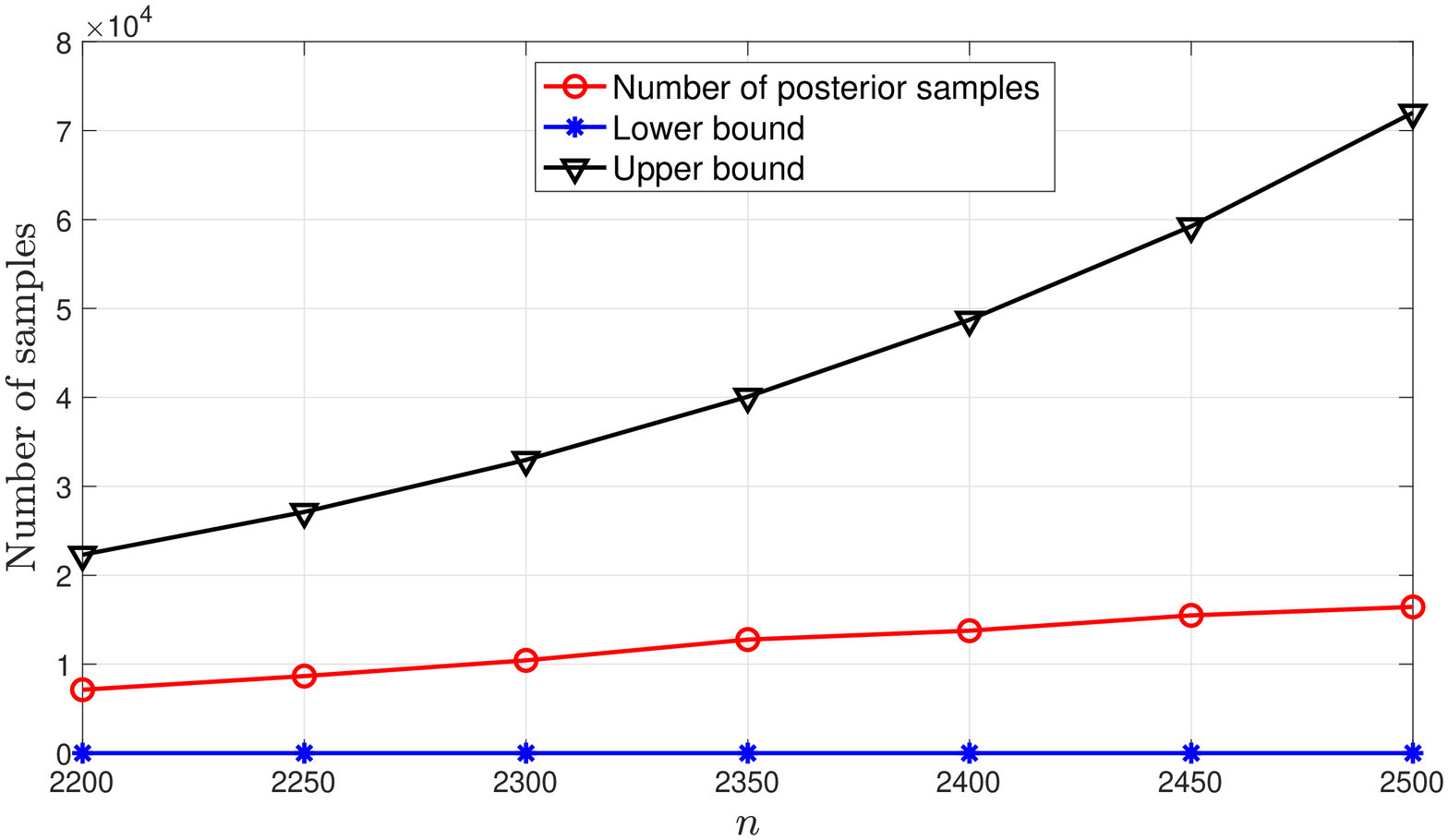} 
		\caption{Average number of posterior samples versus $n$.}
		\label{fig:1}
	\end{minipage}\hfill
	\begin{minipage}{0.49\textwidth}
		\centering
		\includegraphics[width=0.99\textwidth]{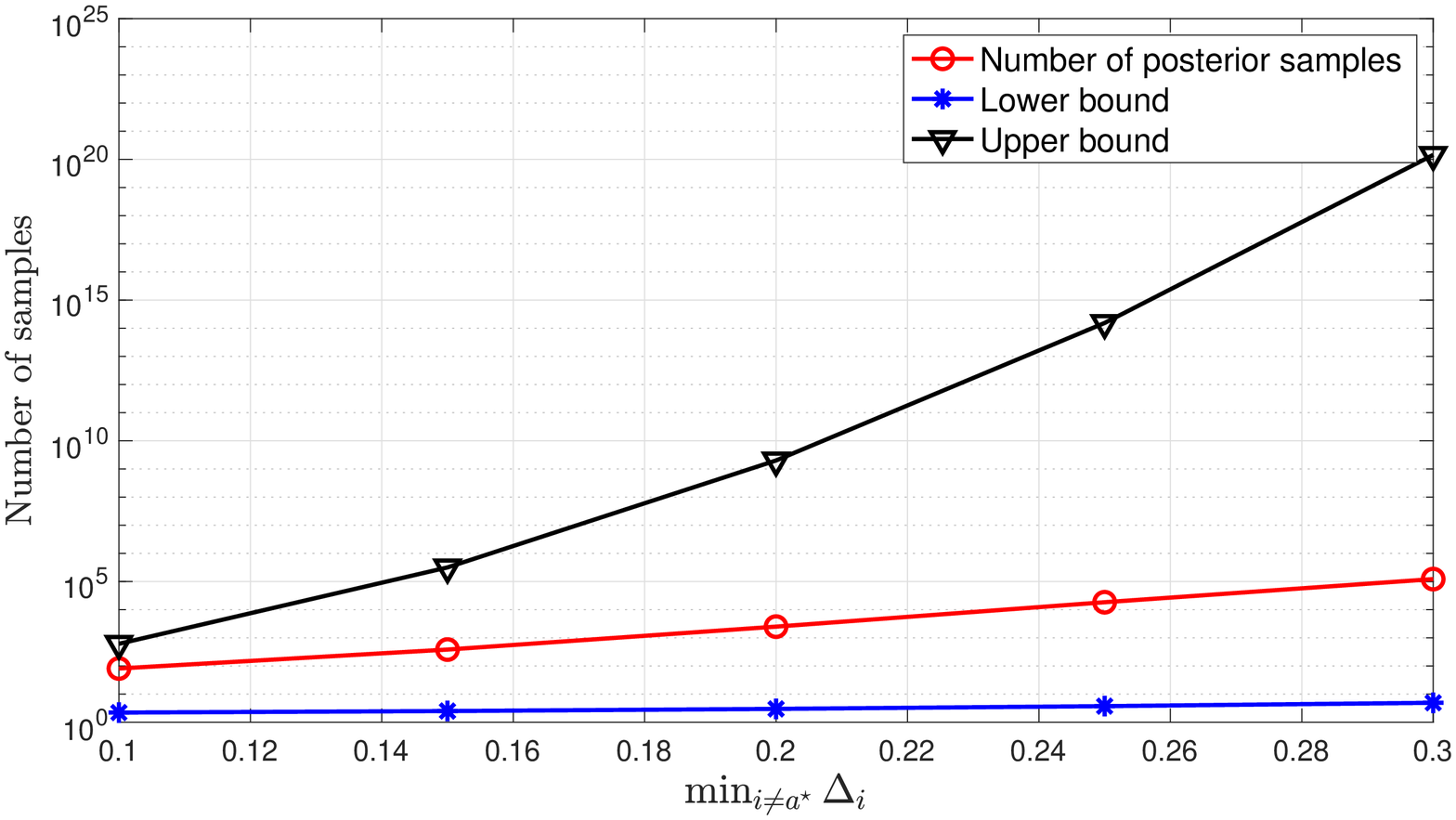} 
		\caption{Average number of posterior samples versus gaps.}
		\label{fig:4}
	\end{minipage}
\end{figure*}

\begin{figure*}[t]
	\centering
	\begin{minipage}{0.49\textwidth}
		\centering
		\includegraphics[width=0.99\textwidth, height = 0.53\textwidth]{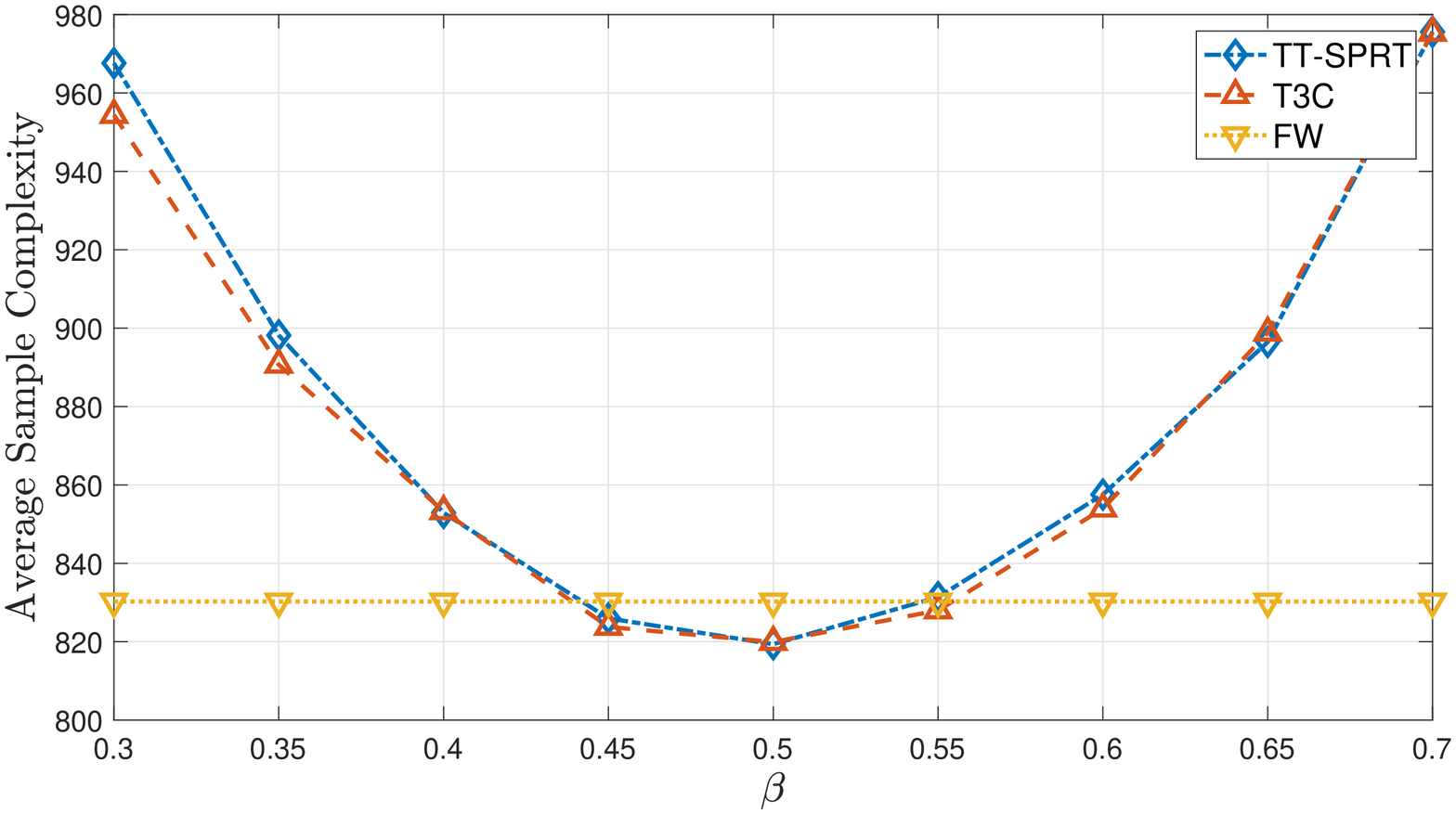} 
		\caption{Sensitivity to $\beta$ (Gaussian): $\bmu = [5, 4.5, 1, 1, 1]$, $\delta = 10^{-15}$}
		\label{fig:g1_1}
	\end{minipage}\hfill
	\begin{minipage}{0.49\textwidth}
		\centering
		\includegraphics[width=0.99\textwidth]{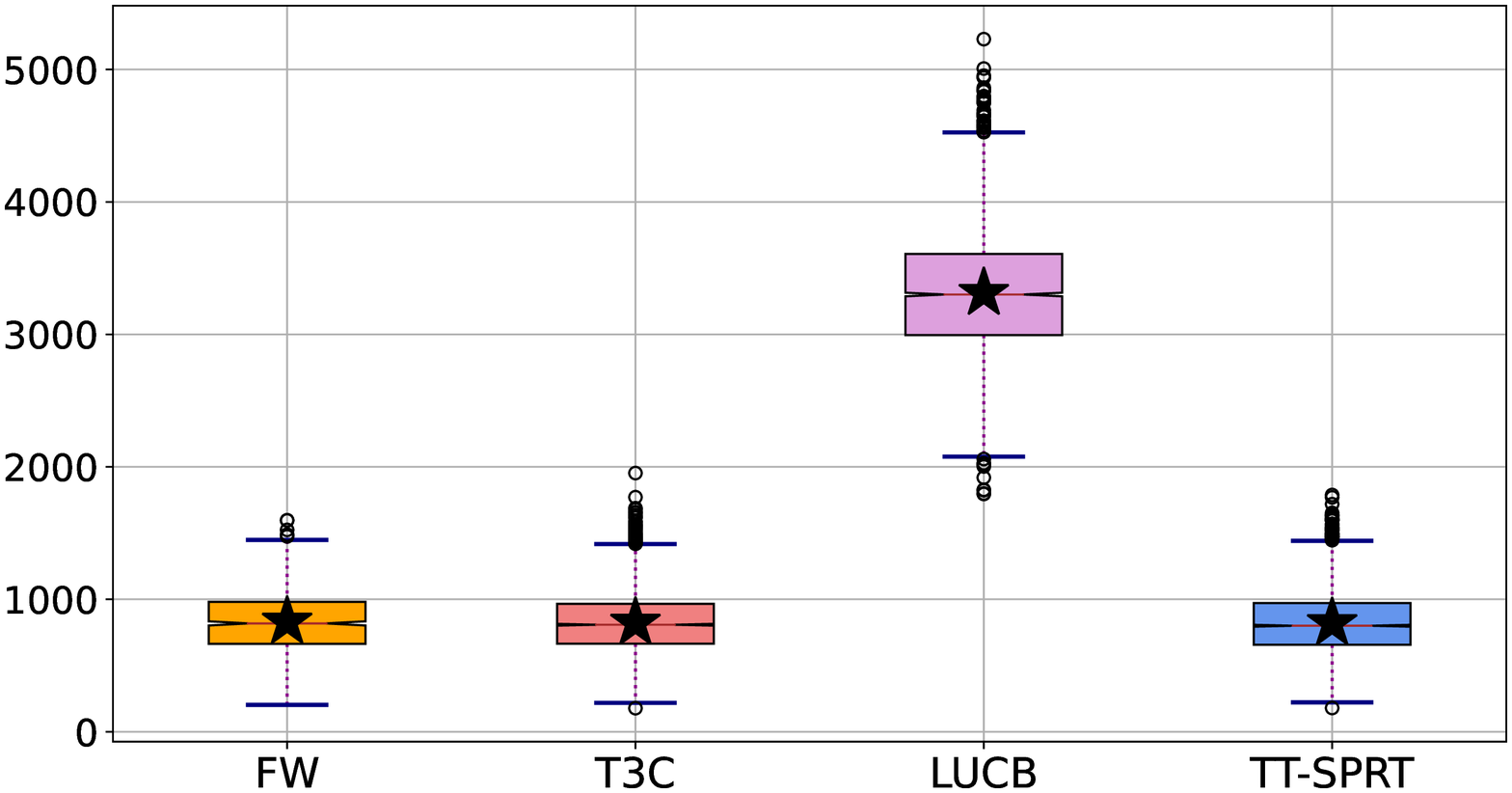} 
		\caption{Sample complexity (Gaussian): $\bmu = [5, 4.5, 1, 1, 1]$, $\delta = 10^{-15}$}
		\label{fig:g1}
	\end{minipage}
\end{figure*}

\subsection{Computational Difficulty of TTTS}

First, to show the computational difficulty of obtaining a challenger in TTTS, in Figure~\ref{fig:1} we evaluate the average number of posterior samples required by TTTS to identify a challenger. As shown analytically in Theorem~\ref{theorem:TTTS_samples}, the expected number of samples required for encountering a challenger scales at least exponentially in $\sqrt{n}$. It is observed from Figure~\ref{fig:1} that as $n$ increases, such average number of posterior samples increases drastically. We also plot the lower bound on the average number of posterior samples obtained in Theorem~\ref{theorem:TTTS_samples}, which matches the scaling behavior of the actual number of samples. Furthermore, Figure~\ref{fig:4} illustrates the scaling behavior of the number of posterior samples against various levels of the minimum gap compared to the best arm, i.e., $\min_{i\neq a^\star}\Delta_i$. For this experiment, we have set $n=500$. It can be readily verified from Theorem~\ref{theorem:TTTS_samples} that as the minimum gap between the means of the best arm and any other arm increases, it becomes more challenging to obtain a challenger from the posterior distribution after it has converged sufficiently. This can be observed in Figure~\ref{fig:4}. Specifically, the number of posterior samples increases at least exponentially at a rate of $\min_{i\neq a^\star}\Delta_i^2$, and it can become extremely large when the gap is large.

\begin{figure*}[t]
	\centering
	\begin{minipage}{0.49\textwidth}
		\centering
		\includegraphics[width=0.99\textwidth, height = 0.53\textwidth]{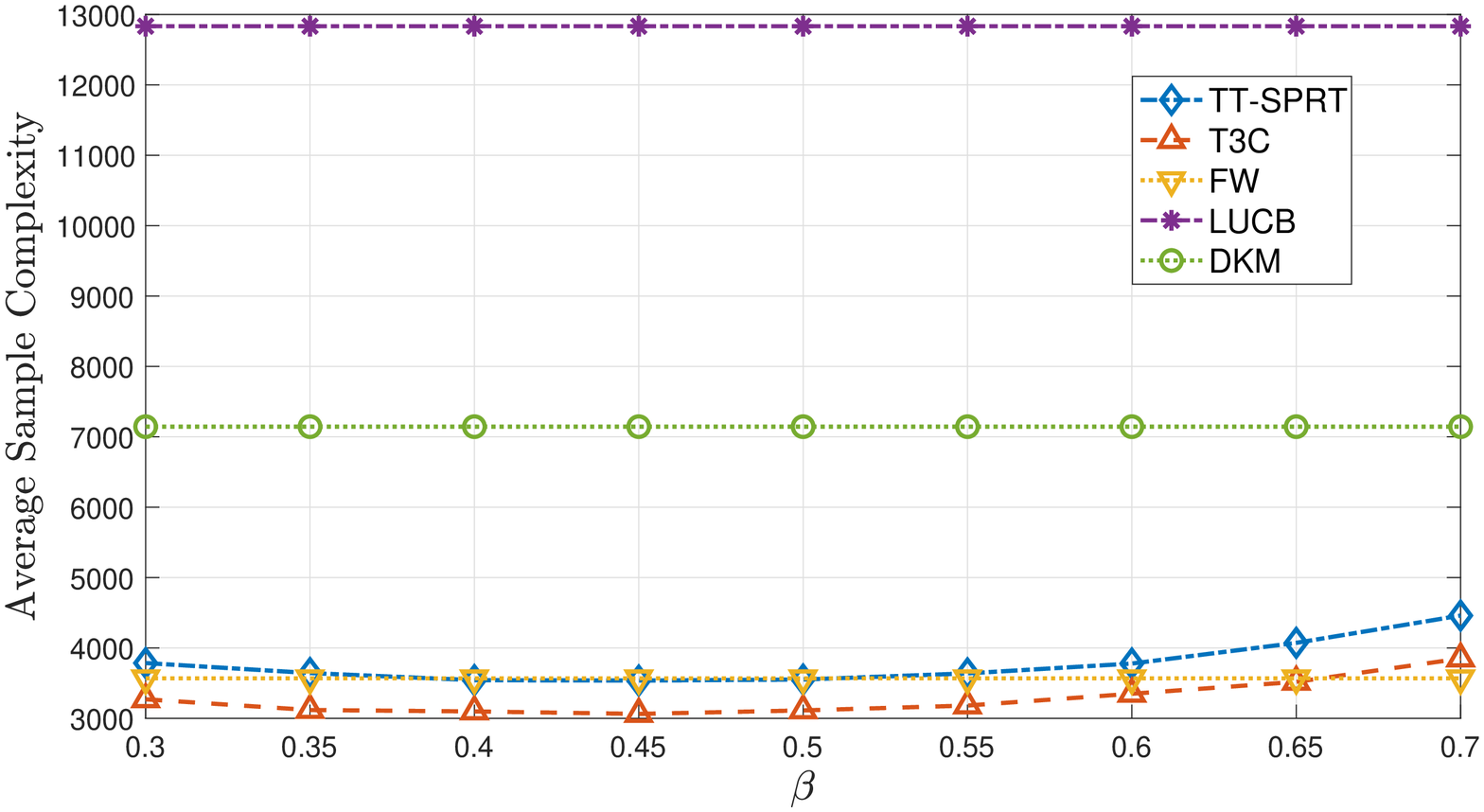} 
		\caption{Sensitivity to $\beta$ (Gaussian): $\bmu = [1, 0.85, 0.8, 0.7]$, $\delta = 0.1$}
		\label{fig:g2_1}
	\end{minipage}\hfill
	\begin{minipage}{0.49\textwidth}
		\centering
		\includegraphics[width=0.99\textwidth]{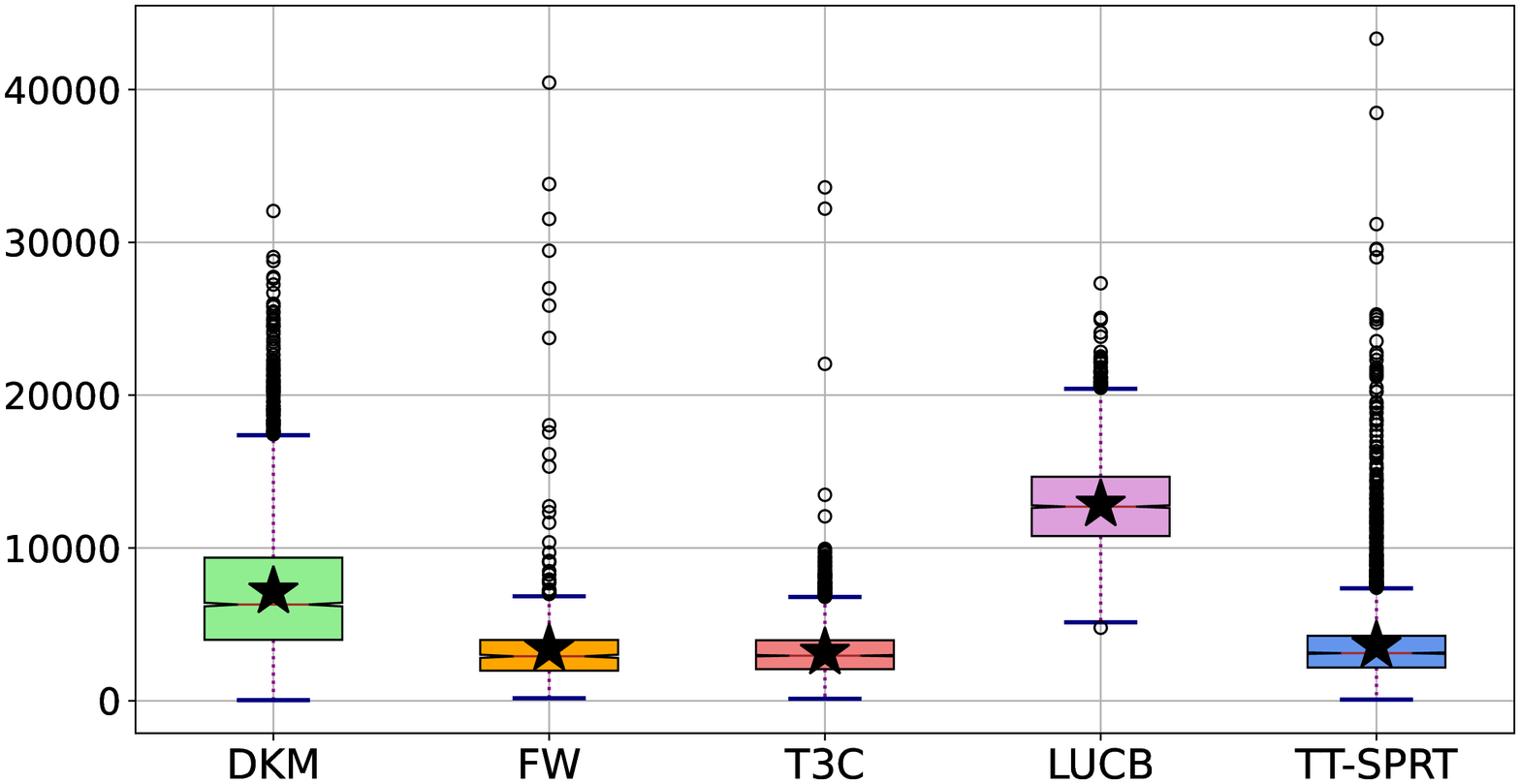} 
		\caption{Sample complexity (Gaussian): $\bmu = [1, 0.85, 0.8, 0.7]$, $\delta = 0.1$}
		\label{fig:g2}
	\end{minipage}
\end{figure*}

\begin{figure*}[t]
	\centering
	\begin{minipage}{0.49\textwidth}
		\centering
		\includegraphics[width=0.99\textwidth, height = 0.53\textwidth]{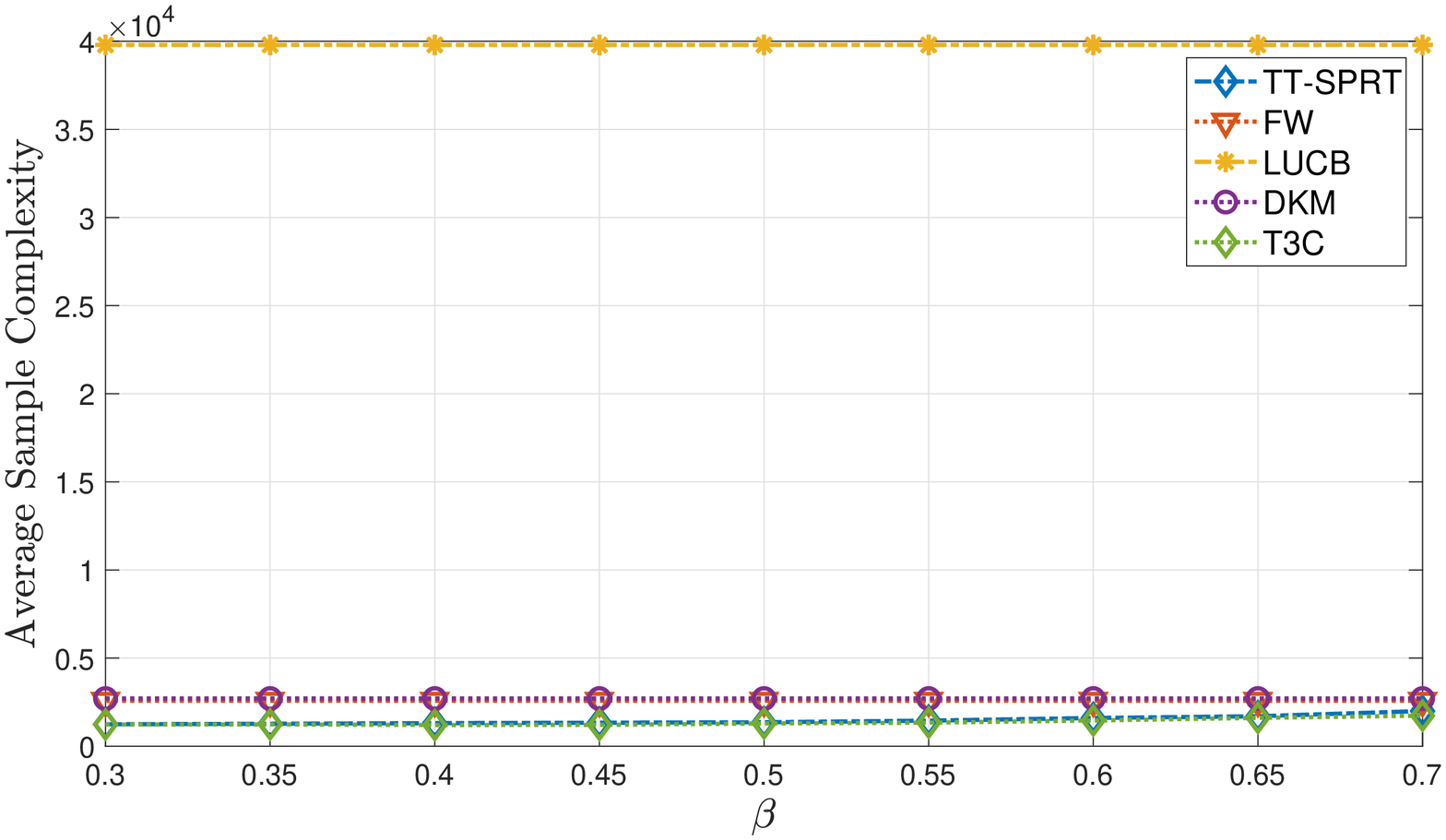} 
		\caption{Sensitivity to $\beta$: (Bernoulli)}
		\label{fig:bernoulli_1}
	\end{minipage}\hfill
	\begin{minipage}{0.49\textwidth}
		\centering
		\includegraphics[width=0.99\textwidth]{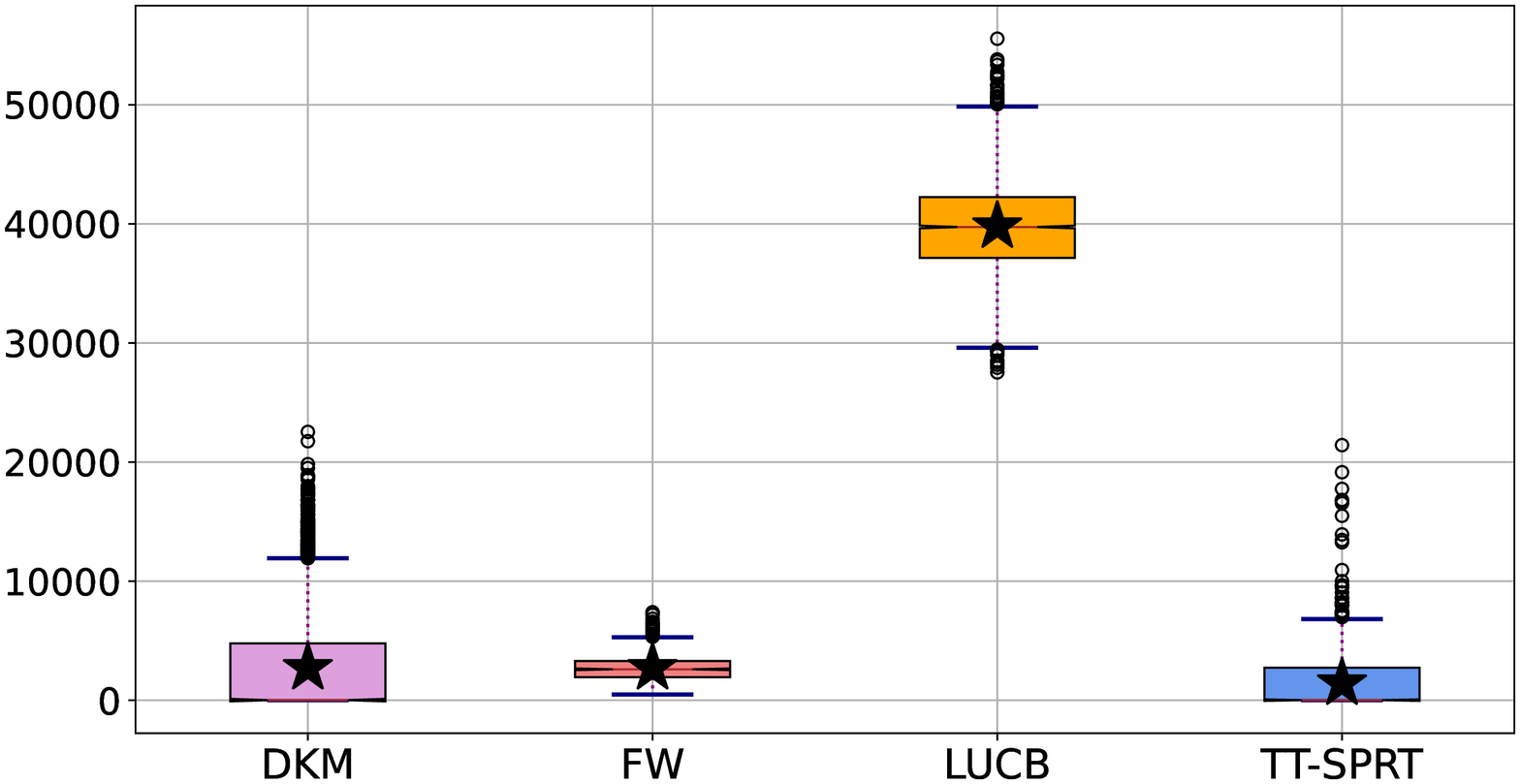} 
		\caption{Sample complexity (Bernoulli)}
		\label{fig:bernoulli}
	\end{minipage}
\end{figure*}

\subsection{Gaussian Bandits}
{Next, we compare the empirical performance of the TT-SPRT algorithm to existing strategies for BAI in the Gaussian bandit setting. Specifically, we compare against four existing BAI algorithms, which are DKM~\cite{degenne2019}, LUCB~\cite{Kalyanakrishnan2012}, T3C~\cite{pmlr-v108-shang20a}, and FW~\cite{FW}. Note that we have not compared our algorithm with TTTS, since TTTS requires an enormous computation time to identify a challenger in our experimental setting. T3C is a computationally efficient alternative to the TTTS algorithm, which has been proposed in~\cite{pmlr-v108-shang20a}. Specifically, T3C is based on posterior sampling for identifying the top arm and replaces the resampling procedure for identifying a challenger in TTTS by defining the challenger based on the minimum transportation cost compared to the best arm. Despite its computational efficiency in selecting a challenger, T3C computes a posterior distribution at each time for identifying the top arm, which may involve Monte Carlo integration without conjugate priors~\cite{russo2016}. The DKM algorithm is based on a gamification principle for algorithm design, which involves a $\bw$ player who plays a sampling distribution and a $\blambda$ player who plays an alternate bandit instance. We have implemented a best-response (zero-regret) $\blambda$ player with the Adahedge $\bw$ player, as described in~\cite{degenne2019}. The LUCB algorithm is based on identifying two arms, a top arm and a challenger, based on the current confidence intervals of the mean estimates. At each iteration, the LUCB algorithm samples both arms and updates their corresponding mean estimates. Finally, the FW algorithm is based on a single iteration of a Frank-Wolfe update step to compute a sampling proportion in each round. }

{We have considered two Gaussian bandit instances, one with $\Delta_{\min}=0$, and the other with $\Delta_{\min}>0$. These instances are characterized by the mean values $\bmu_1 = [5, 4.5, 1,1,1]$, and $\bmu_2 = [1, 0.85, 0.8, 0.7]$. For both instances, we set $\sigma=1$, and $\bmu_2$ is a bandit instance from the experiments in~\cite{degenne2019}. The experiment with $\bmu_1$ is plotted in Figure~\ref{fig:g1_1} and Figure~\ref{fig:g1}, which confirms the superior performance of TT-SPRT over the FW sampling strategy for $0.45\leq\beta\leq 0.55$. Note that we have not plotted DKM and LUCB for this experiment due to their large empirical sample complexities, $3317.8$ and $10,335$ for LUCB and DKM, respectively. The plot corresponding to bandit instance $\bmu_2$ can be found in Figure~\ref{fig:g2_1} and Figure~\ref{fig:g2}. In this experiment, we observe that the empirical sample complexity of the T3C and FW algorithms is slightly better than TT-SPRT. However, this comes at an increased computational cost for the FW sampling strategy for solving a linear program in each iteration. For our implementations in MATLAB, we have used the simplex method for this purpose, which has a worst-case complexity of the order of $O(2^K)$~\cite{deza2008good}. Furthermore, the T3C algorithm requires computing a posterior distribution in each iteration, which may involve Monte Carlo integration if a conjugate prior does not exist~\cite{russo2016}. {Note that the TT-SPRT algorithm uses explicit exploration for the instance $\bmu_1$, since it has $\Delta_{\min}=0$. The effect of explicit exploration (as present in TT-SPRT) versus implicit exploration (as in algorithms such as T3C) is observed to vary, depending on the bandit instance. For example, in $\bmu_1$, T3C is observed to have a comparable performance as TT-SPRT. On the other hand, in $\bmu_2$, we observe that T3C performs slightly better than TT-SPRT. Overall, the necessity of explicit exploration depends on the algorithm design and the bandit instance, and whether or not it improves (or hurts) the sample complexity seems to be unclear.}}

\subsection{Bernoulli Bandits}
{Next, we compare the sample complexity of TT-SPRT in the Bernoulli bandit setting. For comparison, we use a state-of-the-art FW sampling strategy~\cite{FW}, along with existing approaches such as DKM~\cite{degenne2019} and LUCB~\cite{Kalyanakrishnan2012}. For this experiment, we have used the Bernoulli bandit instance from~\cite{degenne2019,pmlr-v49-garivier16a}, with mean values $\bmu = [0.3,0.21,0.2, 0.19, 0.18]$. We have set $\delta=0.1$, and the performance comparison has been plotted in Figure~\ref{fig:bernoulli_1} and Figure~\ref{fig:bernoulli}. Note that there is a difference between the simulations with DKM in~\cite{degenne2019}, even though we use the same bandit instance and value of $\delta$. This is because~\cite{degenne2019} uses a ``stylized threshold'' of $c_{n,\delta} = \log((1+\log n)/\delta)$, which is disallowed by theory. In contrast, we use the theoretically grounded stopping thresholds~\cite{Kaufmann_JMLR}. Figure~\ref{fig:bernoulli} clearly shows that TT-SPRT outperforms the state-of-the-art algorithms. Furthermore, we observe a significantly worse performance of the LUCB algorithm, which is due to the loose confidence bound on the mean estimates prescribed in~\cite{Kalyanakrishnan2012}.}

\begin{figure*}[t]
	\centering
	\begin{minipage}{0.49\textwidth}
		\centering
		\includegraphics[width=0.99\textwidth, height = 0.53\textwidth]{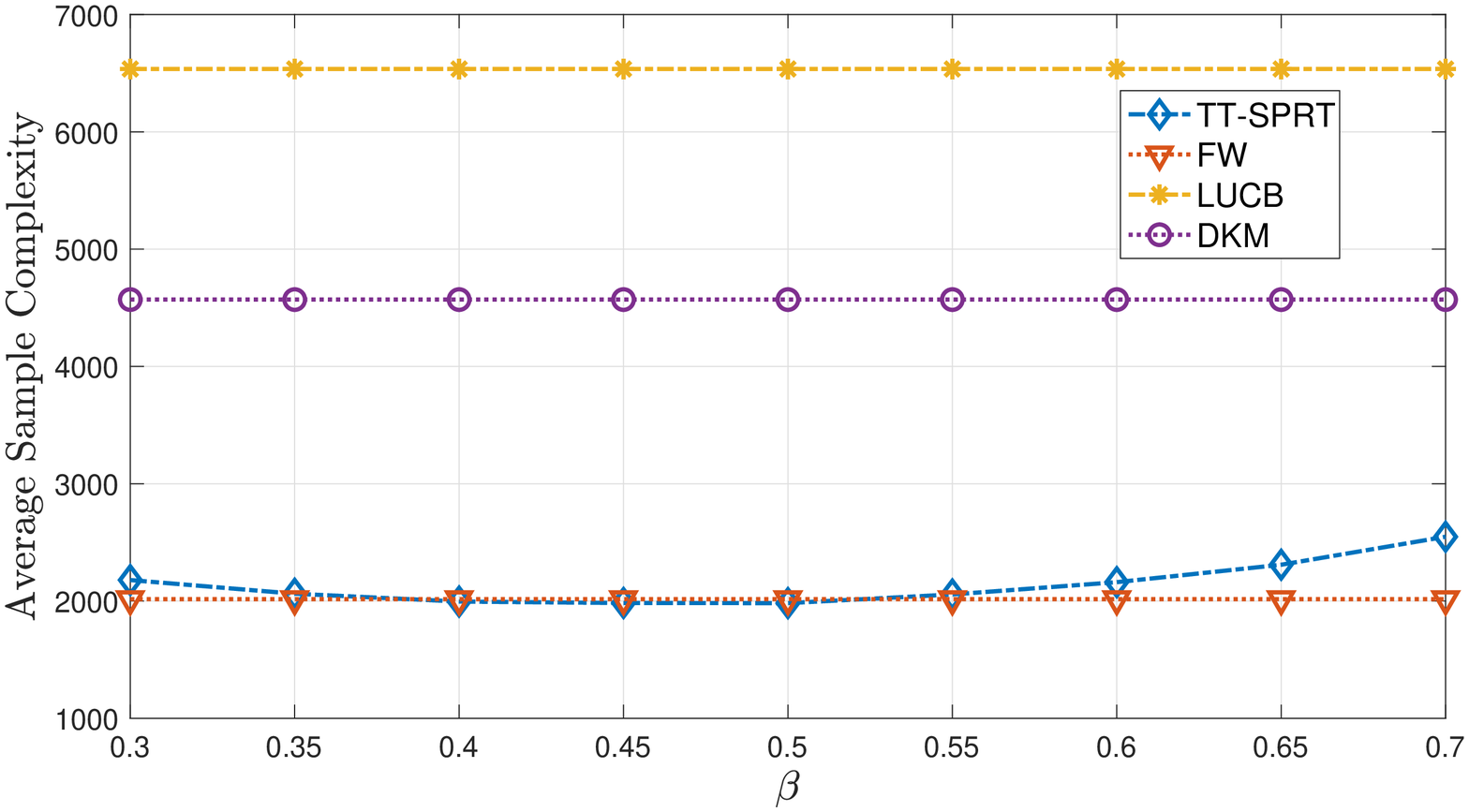} 
		\caption{Sensitivity to $\beta$: (Exponential)}
		\label{fig:exp_1}
	\end{minipage}\hfill
	\begin{minipage}{0.49\textwidth}
		\centering
		\includegraphics[width=0.99\textwidth]{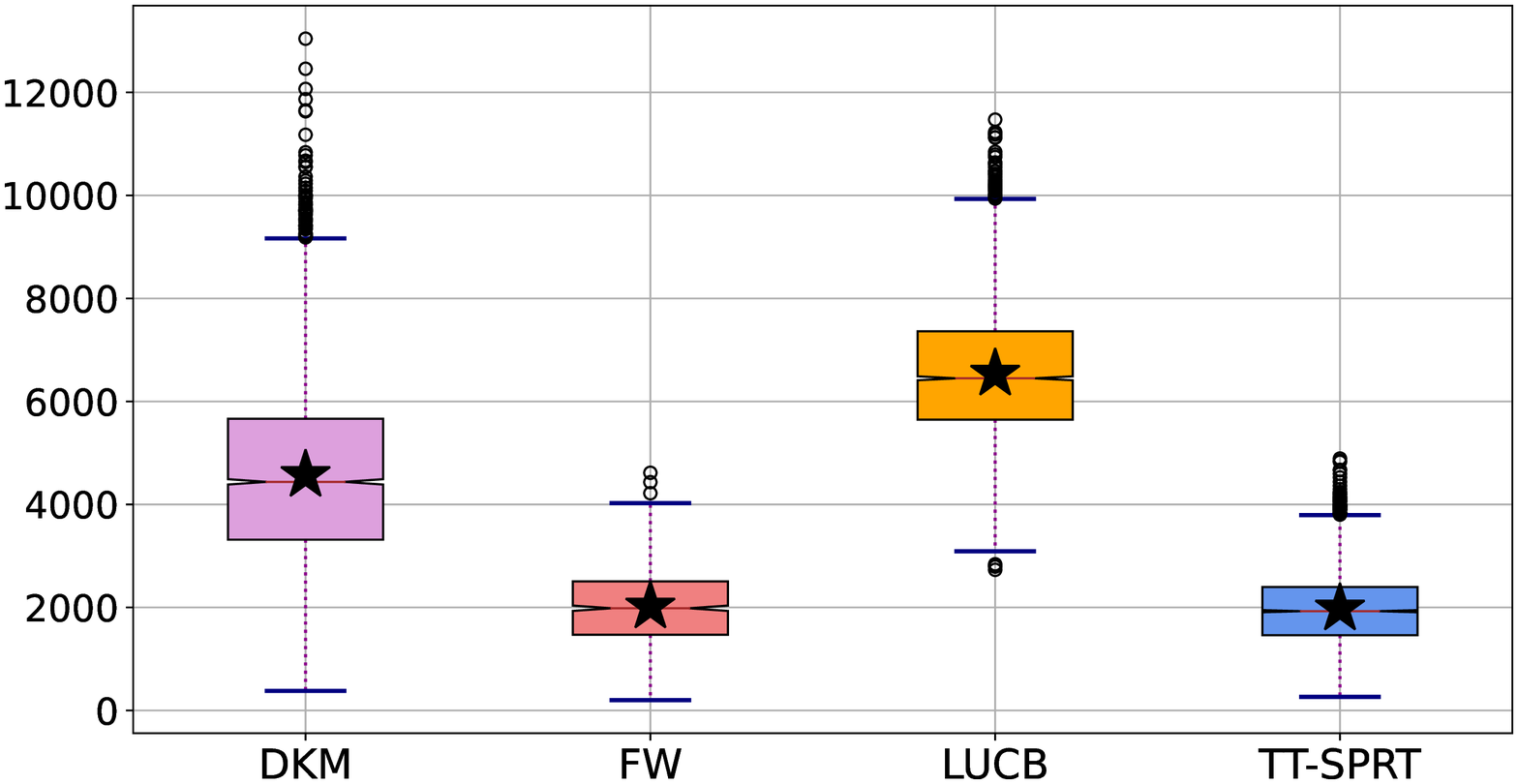} 
		\caption{Sample complexity (Exponential)}
		\label{fig:exp}
	\end{minipage}
\end{figure*}

\subsection{Exponential Bandits}

{Finally, since TT-SPRT generalizes to any member of the single parameter exponential family, we provide an example with an exponential bandit instance with $\bmu = [0.9, 0.7, 0.5, 0.3, 0.1]$. We have set $\delta=0.1$, and the performance of TT-SPRT is compared against that of LUCB~\cite{Kalyanakrishnan2012}, DKM~\cite{degenne2019} and FW~\cite{FW} in Figure~\ref{fig:exp_1} and Figure~\ref{fig:exp}. We can observe that TT-SPRT outperforms DKM and LUCB by a significant margin. Furthermore, its performance is comparable to the FW sampling strategy for a wide range of $\beta$, including the choice of $\beta=0.5$, which has been used to obtain the box plot. Specifically, at $\beta=0.5$, TT-SPRT has an empirical sample complexity of $1982.5$, while that of the FW sampling strategy is $2016.5$. Thus, TT-SPRT has a low empirical sample complexity in the example of exponential bandits, despite having a significantly lower computational cost compared to the state-of-the-art FW sampling rule. Furthermore, in order to highlight the computational burden incurred by the FW sampling rule, we tabulate the computation times required by TT-SPRT and other existing BAI strategies in Table~\ref{tab:1}. These simulations are performed in MATLAB R2022a, on an Apple M1 pro processor equipped with $16$ gigabytes of RAM. Clearly, TT-SPRT outperforms the FW sampling rule in terms of computational time by a wide margin. Note that we have not compared the computation times for TaS, since it is significantly worse (as much as $10$ times that of the FW algorithm~\cite[Appendix G.1]{TTUCB}).}

\begin{table}[h]
	\centering
	\resizebox{0.48\textwidth}{!}{%
		\begin{tabular}{|c|c|c|c|c|c|}
			\hline
			Algorithm & FW     & TT-SPRT  & DKM      & LUCB  &T3C     \\ \hline
			Time      & $8.7984$ & $0.0235$ & $0.0388$ & $0.0218$ & $0.412$ \\ \hline
	\end{tabular}}
	\caption{Average computation time (in seconds)}
	\label{tab:1}
\end{table}

\section{Conclusions}
In this paper, we have investigated the problem of best arm identification (BAI) in stochastic multi-armed bandits (MABs). Arm selection and terminal decision rules are characterized based on generalized likelihood ratio tests with similarities to the conventional sequential probability ratio tests. The decisions rules (dynamic arm selection and stopping time) have two main properties: (1) they achieve optimality in the probably approximately correct learning framework, and, (2) they asymptotically achieve the optimal sample complexity. We have analytically characterized the optimality properties, and comparisons with the state-of-the-art are shown both analytically and numerically. {We have also analyzed the computational challenges in an existing top-two sampling strategy for BAI.}

%\appendices

\appendix

\section{Relevant Lemmas}\label{appendix}
% Throughout the analyses, unless emphasized otherwise, we use $\P$ as the shorthand for $\P_{\bmu}$. Also, throughout the analyses, we denote the probability of sampling any arm $i\in[K]$ at time $n$ by
% % \begin{align}
% %     \psi_{n,i}\triangleq \P\Big( A_n=i\med \mcA_{n-1},\mcX_{n-1}\Big) \ .
% % \end{align}
% \begin{align}
%     \psi_{n,i}\triangleq \P\Big( A_n=i\med \mcF_{n-1}\Big) \ .
% \end{align}
% % \AT{shouldn't this be
% % \begin{align}
% %     \psi_{n,i}\triangleq \P\Big( A_n=i\med \mcF_{n-1}\Big) \ .
% % \end{align}
% % }
% Accordingly, we denote the aggregate probability of selecting arm $i$ up to time $n$ by 
% \begin{align}
%     \Psi_{n,i}\triangleq \sum\limits_{s=1}^n \psi_{s,i}\ .
% \end{align}
%\AT{in notations try to have consistency as much as possible. For instance, if you use $s$ as a time index, try to avoid using others (like $\ell$ that was above before I edited it, and consistently use $s$}

%\AT{given our discussions, is defining poly necessary?}
% Furthermore, we use the notation ${\sf poly}(W_1,W_2,\cdots)\triangleq O(W_1^{a_1}W_2^{a_2}\cdots)$, where $W_1$, $W_2$, $\cdots$ are random variables, and $a_1$, $a_2$, $\cdots$ are constants, which may depend on the parameters of the problem instance, and constants hidden in the $O$ notation.  
\begin{lemma}[Lemma 5,~\cite{TTEI}]\label{lemma:estimator}
In the Gaussian bandit setting, there exists a random variable $W_1$ such that for all $i\in[K]$ and for all $n\in\N$, we almost surely have
\begin{align}
\label{eq:estimator_con}
    |\mu_{n,i}-\mu_i|\leq \sigma W_1\sqrt{\frac{\log(\e + T_{n,i})}{1+T_{n,i}}}\ .
\end{align}
and $\E_{\bmu}[\e^{sW_1}]<+\infty$ for all $s\in\R^+$.
\end{lemma}

\begin{lemma}[Cram\'er-Chernoff inequality for the exponential family]
\label{lemma:Chernoff}

Consider the sequence of i.i.d. random variables $\{X_i:i\in[n]\}$ distributed according to $\P\in\mathscr{P}_b$ with mean $\mu$. Then, for any $\zeta\in\R_+$ and $\bar{X}_n\triangleq \frac{1}{n}\sum_{i=1}^n X_i$ we have
\begin{align}
    \P\left(\bar{X}_n - \mu \geq \zeta\right)\;&\leq\;\exp\left(-n{d_{\sf KL}} (\mu+\zeta\|\mu)\right)\ ,\\
    \label{eq:chernoff_exp}{\rm and}\quad \P\left(\bar{X}_n - \mu \leq -\zeta\right)\;&\leq\;\exp\left(-n{d_{\sf KL}} (\mu-\zeta\|\mu)\right)\ .
\end{align}
\end{lemma}
\begin{proof}
Let us define $S_n\triangleq \sum_{i=1}^n X_i$ and denote the moment generating function associated with $\P$ by $\phi(\lambda)\triangleq \E_{\P}[\e^{\lambda X}]$, which for the single parameter exponential $\mathscr{P}_b$ is given by
\begin{align}
    \phi(\lambda) = \exp\left (b\left ( \dot b^{-1}(\mu) + \lambda\right ) - b\left ( \dot b^{-1}(\mu)\right )\right )\ .
\end{align}
Hence,
\begin{align}
    \P\left( \bar{X}_n - \mu \geq \zeta\right) & = \P\left ( \e^{\lambda S_n}\geq \e^{n\lambda(\mu+\zeta)}\right)\\
    \label{eq:chernoff1}
    &\leq \frac{\E_{\bmu}[\e^{\lambda S_n}]}{\e^{n(\lambda + \zeta)}}\\
    \label{eq:chernoff2}
    & = \displaystyle\prod\limits_{i=1}^n\; \left ( \frac{\E_{\bmu}[\e^{\lambda X_i}]}{\e^{\lambda(\mu+\zeta)}}\right )\\
    & = \left (\frac{\phi(\lambda)}{\e^{\lambda(\mu+\zeta)}}\right )^n \ ,
    \label{eq:chernoff3}
\end{align}
where~(\ref{eq:chernoff1}) is holds due to Markov's inequality and~(\ref{eq:chernoff2}) holds due to independence. Since~(\ref{eq:chernoff3}) is valid for any $\lambda\in\R_+$, we have
\begin{align}
    \P\left( \bar{X}_n - \mu \geq \zeta\right)\;\leq\;\left (\inf\limits_{\lambda\in\R_+}\; \frac{\phi(\lambda)}{\e^{\lambda(\mu+\zeta)}}\right )^n\ .
    \label{eq:chernoff4}
\end{align}
It can be readily verified that the infimum of the right-hand-side in~(\ref{eq:chernoff4}) is obtained at
\begin{align}
   \arginf \limits_{\lambda\in\R_+}\; \frac{\phi(\lambda)}{\e^{\lambda(\mu+\zeta)}}= \dot b^{-1}(\mu + \zeta) - \dot b^{-1}(\mu)\ .
    \label{eq:chernoff5}
\end{align}
Using~(\ref{eq:chernoff5}), we can rewrite~(\ref{eq:chernoff4}) as
\begin{align}
    \P\left( \bar{X}_n - \mu \geq \zeta\right)\;&\leq\; \exp\left ( -n\left[ b(\dot b^{-1}(\mu)) - b(\dot b^{-1}(\mu+\zeta)) - (\mu+\zeta)\left ( \dot b^{-1}(\mu) - \dot b^{-1}(\mu + \zeta)\right )\right]\right )\\
    \label{eq:exponential_D_KL}& = \exp\left ( -n{\sf D_{KL}(\mu+\zeta\|\mu)}\right )\ .
\end{align}
where \eqref{eq:exponential_D_KL} holds due to~(\ref{eq:KL}). The proof of \eqref{eq:chernoff_exp} follows a similar line of arguments. 
\end{proof}

\begin{lemma}\label{lemma:KL_conv}
For any pair of distributions $\P,\P^\prime\in\mathscr{P}_b$ with mean values $\mu,\mu^\prime$, for any $\epsilon\in\R_+$,
% For any $\P\in\mathscr{P}$ with mean $\mu$ and any $\P^\prime\in\mathscr{P}$ with mean $\mu^\prime$, and for any $\epsilon>0$, we have
\begin{align}
    &{d_{\sf KL}} (\mu+\epsilon\|\mu^\prime) - {d_{\sf KL}} (\mu\|\mu^\prime) = O(\epsilon)\ ,\\
    \text{and}\quad &{d_{\sf KL}} (\mu-\epsilon\|\mu^\prime) - {d_{\sf KL}} (\mu\|\mu^\prime) = O(\epsilon)\ .
    \label{eq:KL_conv_negative}
\end{align}
\end{lemma}
\begin{proof}
Note that $b : \Theta\mapsto \R$ is a convex, twice-differentiable function over a compact space $\Theta$, and hence it is Lipschitz continuous. Let us define a Lipschitz constant, based on which
\begin{align}
\label{eq:Lipschitz}
    |b(\theta) - b(\theta^\prime)| \leq \mcL_1|\theta - \theta^\prime|\ ,\quad\forall\theta,\theta^\prime\in\Theta\ .
\end{align}
Corresponding to any $\theta\in\Theta$, $\Ddot{b}(\theta)$ is the variance of the distribution $\P_{\theta}\in\mathscr{P}_b$, and hence, $\Ddot{b}(\theta)>0$.
% Next, note that for any $\theta\in\Theta$ and the corresponding measure $\P_{\theta}\in\mathscr{P}$, $\Ddot{b}(\theta) = \int\limits_{x\in\R} (x-\dot{b}(\theta))^2\diff \P_{\theta}(x)$, and hence $\Ddot{b}(\theta)>0$. 
Let us define
\begin{align}
    \mcL_2\triangleq \left (\min\limits_{\theta\in\Theta} \Ddot{b}(\theta)\right )^{-1}\ .
\end{align}
Using the mean value theorem, for any pair $\theta,\theta^\prime\in \Theta$, there exists a $\lambda\in\Theta$ for which we have:
\begin{align}
    \frac{\dot{b}(\theta) - \dot{b}(\theta^\prime)}{\theta - \theta^\prime} &= \Ddot{b}(\lambda) \geq \frac{1}{\mcL_2}\ .
    \label{eq:reverse_Lipschitz1}
\end{align}
Using~(\ref{eq:reverse_Lipschitz1}), for any pair $z,z^\prime\in\R$, we have
\begin{align}
    |\dot{b}^{-1}(z) - \dot{b}^{-1}(z^\prime)|\leq \mcL_2 |z-z^\prime|\ .
    \label{eq:reverse_Lipschitz2}
\end{align}
This indicates that for any $\epsilon\in\R_+$, we have
\begin{align}
    &{d_{\sf KL}} (\mu+\epsilon\|\mu^\prime) - {d_{\sf KL}} (\mu\|\mu^\prime)\nonumber\\
    &\quad \stackrel{(\ref{eq:KL})}{=} b(\dot{b}^{-1}(\mu)) - b(\dot{b}^{-1}(\mu+\epsilon)) + \mu \left[\dot{b}^{-1}(\mu+\epsilon) - \dot{b}^{-1}(\mu)\right] + O(\epsilon)\\
    \label{eq:KL_conv1}
    &\quad\stackrel{(\ref{eq:Lipschitz})}{\leq} \mcL_1\left\lvert\dot{b}^{-1}(\mu) - \dot{b}^{-1}(\mu+\epsilon)\right\rvert + \mu \left[\dot{b}^{-1}(\mu+\epsilon) - \dot{b}^{-1}(\mu)\right] + O(\epsilon)\\
    \label{eq:KL_conv2}
    &\quad \stackrel{(\ref{eq:reverse_Lipschitz2})}{\leq} \left(\mcL_1+\mu\right)\mcL_2\epsilon + O(\epsilon)\\
    &\quad = O(\epsilon)\ .
\end{align}
Proving (\ref{eq:KL_conv_negative}) follows similar steps.  

\end{proof}

\section{Proof of Lemma~\ref{theorem:explore_gaussian}}\label{proof:exploration_Gaussian}
For a given $L>0$ and for all $n\in\N$, define the sets %\AT{change these to $U_n(L),V_n(L)$}
\begin{align}
    &U_n(L)\triangleq \{i\in[K] \; : \; T_{n,i}\leq\sqrt{L} \}\ ,\\
    \text{and}\quad & V_n(L) \triangleq \{i\in[K] \; : \; T_{n,i}\leq L^{3/4} \}\ .
\end{align}
First, we establish the fact that for any $L>0$, if the set $U^L_n\neq\emptyset$, then at least one of the top two contenders $a_n^1$ or $a_n^2$ is contained in $V_n(L)$, i.e., either $a_n^1\in V_n(L)$, or $a_n^2\in V_n(L)$, or $\{a_n^1,a_n^2\}\in V_n(L)$. This is captured in the following Lemma. %\AT{write the conditioin mathematically}. %\AT{this is not clear. Do you mean that exactly one is under-sampled, or at least one of them is under-sampled, or at most one of them is under-sampled? Change the language to one of these. It seems the middle one is what's proved in the next lemma.} 
%\AT{the following lemma is not showing being under-sampled. The definitions of these sets is not under-sampled, and showing that the contenders belong to them is not proving being under-sampled.}

\begin{lemma}\label{lemma:explore_1_or_2}
If $U_n(L)\neq\emptyset$, then there exists a random variable $L_3$ such that for all $L>L_3$, either $a_n^1\in V_n(L)$ or $a_n^2\in V_n(L)$. Furthermore, $\E_{\bmu}[L_3]<+\infty$.
\end{lemma}
\begin{proof}
If $a_n^1\in V_n(L)$, then the claim is proved. Let us assume that $a_n^1\in \overline{V_n(L)}$, where $\overline{V_n(L)}\triangleq[K]\setminus V_n(L)$. We will prove that in this case, $a_n^2\in V_n(L)$. Let us assume otherwise, that is, $a_n^2\in\overline{V_n(L)}$. Furthermore, define 
\begin{align}
    I_n^\star \triangleq \argmin\limits_{i\in U_n(L)} \Lambda_n(a_n^1,i)\ .
\end{align}
% We will show that for all $j\in\overline{V_n(L)}\setminus\{a_n^1\}$,
% \begin{align}\label{eq: step1_1}
%     \Lambda_n(a_n^1,I_n^\star) < \Lambda_n(a_n^1,j)\ ,
% \end{align}
%\AT{It's better to change the above to:
%}
We will show that 
\begin{align}\label{eq: step1_1_new}
    \Lambda_n(a_n^1,I_n^\star) < \Lambda_n(a_n^1,a_n^2)\ ,
\end{align}
which contradicts the definition of $a_n^2$, and hence our assumption that $a_n^2\in \overline{V_n(L)}$. 
%\AT{I don't think using WLOG is relevant here. Furthermore, (83) has no $i$ in it, and what you've written also is not having a correct reference. This should be revised.} Without loss of generality, we prove~(\ref{eq: step1_1}) for $i=I_n^\star$, while it can be proved for any $i\in U_n(L)$, following the same line of arguments. 
To prove (\ref{eq: step1_1_new}), we show the following more general property that for all $j\in\overline{V_n(L)}\setminus\{a_n^1\}$
\begin{align}\label{eq: step1_1}
    \Lambda_n(a_n^1,I_n^\star) < \Lambda_n(a_n^1,j)\ ,
\end{align}
Note that from the definition of GLLRs for Gaussian bandits in~(\ref{eq:gaussian_LLR}), (\ref{eq: step1_1}) is equivalent to: 
%\AT{why? you need to provide reference to the properties you're using to show this equivalence.}
\begin{align}\label{eq: step1_2}
    &\frac{T_{n,I_n^\star}}{T_{n,a_n^1}+T_{n,I_n^\star}}(\mu_{n,a_n^1}-\mu_{n,I_n^\star})^2 < \frac{T_{n,j}}{T_{n,a_n^1}+T_{n,j}}(\mu_{n,a_n^1}-\mu_{n,j})^2\ .
\end{align}
Using Lemma~\ref{lemma:estimator}, there exists a random variable $W_1$ such that for all $i\in\overline{V_n(L)}$ and for all $n\in\N$, with probability $1$ we have
%\AM{First upper bound RHS, the invoke LEmma3}
\begin{align}\label{eq:i}
    |\mu_{n,i} - \mu_i| &\leq 2\sigma W_1 \sqrt{\frac{\log(\e+T_{n,i})}{1+T_{n,i}}}
    \leq 2\sigma W_1 \sqrt{\frac{\log(\e+L^{3/4})}{1+L^{3/4}}}\ ,
\end{align}
where (\ref{eq:i}) holds due to the fact that for $i\in\overline{V_n(L)}$ we have $T_{n,i}> L^{3/4}$ and that $\frac{\log(\e + x)}{1+x}$ is a decreasing function of $x$. Furthermore, again since $\frac{\log(\e + x)}{1+x}$ is a decreasing function, there always exists a random variable $L_1$ as a function of $W_1$ and instance-dependent parameters, such that 
%\AT{this statement -- which is paraphrasing what you had earlier -- is not fully accurate. $W_1$ is random and it must be emphasized that $L_1$ is also a function of it and therefore random. Once that's adjusted, you need to be also careful about what the following equation exactly means.}
\begin{align}
\label{eq:i_i}
    \sqrt{\frac{\log(\e+L_1^{3/4})}{1+L_1^{3/4}}} \leq \frac{\Delta_{\min}}{4\sigma W_1}\ ,
\end{align}
where we have defined $\Delta_{\min}\triangleq \min_{i,j\in[K]}|\mu_i - \mu_j|$.  Hence, combining~(\ref{eq:i}) and~(\ref{eq:i_i}), for all $L>L_1$, we obtain %\AT{where does it come from? again it seems you're using a property that you're not referring to.}
\begin{align}
   |\mu_{n,i} - \mu_i|&\leq 2\sigma W_1\cdot\frac{\Delta_{\min}}{4\sigma W_1} =\frac{\Delta_{\min}}{2}\ .
    \label{eq:step1_3}
\end{align}
%\AM{move (65) above}
Furthermore, for $L>L_1$, with probability $1$ we have
\begin{align}\label{eq:mu_LHS}
    |\mu_{n,a_n^1} - \mu_{n,I_n^\star}|& = |\mu_{n,a_n^1} - \mu_{a_n^1} + \mu_{I_n^\star} - \mu_{n,I_n^\star} + \mu_{a_n^1} - \mu_{I_n^\star}|\\
    &\leq |\mu_{n,a_n^1} - \mu_{a_n^1}| + |\mu_{n,I_n^\star} - \mu_{I_n^\star}| + |\mu_{a_n^1} - \mu_{I_n^\star}|\\
    &\leq \frac{\Delta_{\min}}{2} + 2\sigma W_1 + \Delta_{\max}\ ,
    \label{eq:mu_LHS1}
\end{align}
where we have defined $\Delta_{\max}\triangleq \max_{i,j\in[K]}(\mu_i - \mu_j)$, and (\ref{eq:mu_LHS1}) follows from (\ref{eq:step1_3}), along with invoking Lemma~\ref{lemma:estimator}, based on which for any $i\in U_n(L)$ we have $|\mu_{n,i} - \mu_i|\leq 2\sigma W_1$ with probability $1$. Furthermore, using~Lemma~\ref{lemma:estimator}, for any two arms $i,j\in\overline{V_n(L)}$, where without the loss of generality $\mu_i>\mu_j$, for $L>L_1$, with probability $1$ we have 
\begin{align}\label{eq:mu_RHS}
    \mu_{n,i} - \mu_{n,j} &\geq \mu_i - \mu_j - \sigma W_1\sqrt{\frac{\log(\e + T_{n,i})}{1+T_{n,i}}}- \sigma W_1\sqrt{\frac{\log(\e + T_{n,j})}{1+T_{n,j}}}\\
    \label{eq:mu_RHSii}
    &\geq \mu_i - \mu_j - 2\sigma W_1\sqrt{\frac{\log(\e + L^{3/4})}{1+L^{3/4}}}\\
    \label{eq:iv}
    &\geq \Delta_{\min} - 2\sigma W_1\cdot\frac{\Delta_{\min}}{4\sigma W_1}\\
    &=\frac{\Delta_{\min}}{2}\ ,
    \label{eq:mu_RHS_1}
\end{align}
where~(\ref{eq:mu_RHS}) follows from Lemma~\ref{lemma:estimator}, (\ref{eq:mu_RHSii}) uses the fact that $i,j\in\overline{V_n(L)}$ combined with the fact that $\frac{\log(\e + x)}{1+x}$ is a decreasing function of $x$, and (\ref{eq:iv}) holds since $L>L_1$, for the random variable $L_1$ that satisfies~(\ref{eq:i_i}).
%where (\ref{eq:iv}) holds by choosing $L>L_2$, where $L_2={\sf poly}(W_1)$.
Next, we will prove a property that is a sufficient condition for ~(\ref{eq: step1_2}). Specifically,
\begin{align}
    \left(\frac{T_{n,a_n^1}}{T_{n,I_n^\star}} + 1\right)\frac{\Delta_{\min}^2}{4\left(\frac{\Delta_{\min}}{2} + 2\sigma W_1 + \Delta_{\max}\right)^2} > \frac{T_{n,a_n^1}}{T_{n,j}} + 1 \ ,
    \label{eq:explore_to_prove}
\end{align}
which is obtained by upper-bounding the ratio $(\mu_{n,a_n^1}-\mu_{n,I^\star_n})^2/(\mu_{n,a_n^1-\mu_{n,j}})^2$ using~(\ref{eq:mu_LHS1}) and~(\ref{eq:mu_RHS_1}). Next, we will prove that~(\ref{eq:explore_to_prove}) holds.
\begin{enumerate}
    \item \textbf{If $T_{n,a_n^1}/T_{n,j} < 1$:} In this case, note that the left-hand-side of (\ref{eq:explore_to_prove}) can be lower-bounded as
    \begin{align}
        \left(\frac{T_{n,a_n^1}}{T_{n,I_n^\star}} + 1\right)\frac{\Delta_{\min}^2}{4\left(\frac{\Delta_{\min}}{2} + 2\sigma W_1 + \Delta_{\max}\right)^2} & > \frac{T_{n,a_n^1}}{T_{n,I_n^\star}}\frac{\Delta_{\min}^2}{4\left(\frac{\Delta_{\min}}{2} + 2\sigma W_1 + \Delta_{\max}\right)^2}\\
        &> \frac{L^{3/4}}{\sqrt{L}}\frac{\Delta_{\min}^2}{4\left(\frac{\Delta_{\min}}{2} + 2\sigma W_1 + \Delta_{\max}\right)^2}\ ,
        \label{eq:explore_case1}
    \end{align}
    where~(\ref{eq:explore_case1}) follows from the assumption that $a_n^1\in\overline{V_n(L)}$, and from the definition of $I_n^\star$. Using the lower-bound in~(\ref{eq:explore_case1}), we want to find an $L$ that satisfies the following condition, which is stronger compared to~(\ref{eq:explore_to_prove}). 
    \begin{align}
        \frac{L^{3/4}}{\sqrt{L}}\frac{\Delta_{\min}^2}{4\left(\frac{\Delta_{\min}}{2} + 2\sigma W_1 + \Delta_{\max}\right)^2} > 2\ .
    \end{align}
    Let us define
    \begin{align}
        L_2\triangleq \left[\frac{8\left(\frac{\Delta_{\min}}{2} + 2\sigma W_1 + \Delta_{\max}\right)^2}{\Delta_{\min}^2}\right]^4\ .
        \label{eq:L4}
    \end{align}
    Choosing $L>L_2$, it can be readily verified that the condition in~(\ref{eq:explore_to_prove}) is satisfied, which implies that the condition in~(\ref{eq: step1_1}) is satisfied for all $L>L_2$.
    \item \textbf{If $T_{n,a_n^1}/T_{n,j}\geq 1$: } In this case, we will show the following property, which is sufficient to ensure~(\ref{eq:explore_to_prove}):
    \begin{align}
        \frac{T_{n,a_n^1}}{T_{n,I_n^\star}}\cdot \frac{\Delta_{\min}^2}{4\left(\frac{\Delta_{\min}}{2} + 2\sigma W_1 + \Delta_{\max}\right)^2} > \frac{2T_{n,a_n^1}}{T_{n,j}}\ ,
    \end{align}
    which also holds by choosing $L>L_2$, where $L_2$ is defined in~(\ref{eq:L4}).
\end{enumerate}
Finally, by defining $L_3\triangleq \max\{L_1,L_2\}$, we have shown that~(\ref{eq: step1_1}) is satisfied for all $L>L_3$. Furthermore, since $L_3$ is a function of $W_1$, leveraging Lemma~\ref{lemma:estimator} it satisfies $\E_{\bmu}[L_3]<+\infty$.
\end{proof}
% \begin{lemma}[Sufficient Exploration]\label{lemma:exploration}
% In the Gaussian bandit setting, under the TT-SPRT sampling rule, there exists $N_1={\sf poly}(W_1)$ such that for all $n>N_1$ and $\forall i\in[K]$, $T_{n,i}\geq\sqrt{n/K}$ almost surely.
% \end{lemma}
%\begin{proof}
%\noindent\textbf{Proof of Lemma~\ref{lemma:exploration}}\\
Finally, we will prove Lemma~\ref{theorem:explore_gaussian}. Leveraging Lemma~\ref{lemma:explore_1_or_2}, we have that for any $L>L_3$, if $U_n(L)\neq\emptyset$, then either $a_n^1\in V_n(L)$ or $a_n^2\in V_n(L)$. Furthermore, leveraging Lemma 11 in~\cite{pmlr-v108-shang20a}, we can show that if Lemma~\ref{lemma:explore_1_or_2} holds, then there exists a random variable $L_0$, which is a function on $W_1$ and instance-dependent parameters, such that for all $L>L_0$, $U_{\lfloor KL \rfloor}(L)=\emptyset$. Finally, defining $N_1\triangleq KL_0$, we obtain that for all $n>N_1$, $T_{n,i}/n\geq\sqrt{n/K}$ for every $i\in[K]$. Furthermore, since the random variable $N_1$ is a function of $W_1$ and instance-dependent parameters, leveraging Lemma~\ref{lemma:estimator}, we obtain that $\E_{\bmu}[N_1]<+\infty$. This concludes the proof.

%\AT{this whole proof needs to be revisited and tightened. References to poly seem unnecessary. Randomness of the $L$ values should be accounted for carefully. In some places, there is the need for adding more references to the properties you're using.}

\section{Proof of Lemma~\ref{theorem:mean:convergence:exp} and Lemma~\ref{theorem:mean:convergence:gaussian}}
\label{proof:mean:convergence:exp}
%\noindent\textbf{Proof of Theorem~{cross-ref Theorem 4}}:

% \begin{lemma}[Convergence in Mean]
% \label{lemma:suf_explore_bern}
% For any $\epsilon>0$, there exists $N_\epsilon \in\N$, $\E_{\bmu}[N_{\epsilon}^\mu]<+\infty$, such that for all $n>N_\epsilon$ and for all $i\in[K]$,
% \begin{align}
%     |\mu_{n,i} - \mu_i | \leq \epsilon \;\;\rm{almost\; surely}\ .
% \end{align}
% \end{lemma}
\subsection{Exponential Family of Bandits}
%The first part is a direct consequence of the TT-SPRT arm selection strategy in~(\ref{eq:sampling rule_B}). 
%To show this, note that at any time instant $M\in\N$, we have the following two cases:
% \begin{enumerate}
%     \item $\mcI_M = \emptyset$: In this case, we will prove that for all $n>M$ and for all $i\in[K]$, $T_{n,i} > \sqrt{n/2K}$. We will show this using induction. At time $M$, we have assumed that $\mcI_M = \emptyset$, which implies that $T_{M,i} > \sqrt{M/2K}$. Next, let us assume that for any $n>M$, $T_{n,i} > \sqrt{n/2K}$. We will show that $T_{}$
% \end{enumerate}
% Next, let us define
% \begin{align}
%     N_{\epsilon}^\mu \triangleq \inf \left\{ N\in\N : |\mu_{n,i}-\mu_i| \leq \epsilon,\;\forall i\in[K],\;\forall n>N_{\epsilon}^\mu\right\}\ .
% \end{align}
%\subsection{Exponential Family of Bandits}
For any $t\in\N$, for a given $\epsilon\in\R_+$ and any bandit realization with mean $\bmu$ we have
\begin{align}
    \P_{\bmu}(N_{\epsilon}^\mu>t) & =\sum\limits_{s=t+1}^\infty \P_{\bmu}(N_{\epsilon}^\mu = s)\\
    & = \sum\limits_{s=t+1}^{\infty} \P_{\bmu}\left( \exists i \in[K] : |\mu_{s-1,i}-\mu_i|>\epsilon\ ,\;{\rm and},\; \forall u\geq s, \forall i\in[K], |\mu_{u,i}-\mu_i|\leq\epsilon\right)\\
    \label{eq:conv_mean_exp1}
    & \leq\sum\limits_{s = t}^\infty \P_{\bmu}\left ( \exists i\in[K] : |\mu_{s,i} - \mu_i| > \epsilon\right)\\
    \label{eq:conv_mean_exp2}
    & \leq \sum\limits_{i\in[K]} \sum\limits_{s = t}^\infty \P_{\bmu}\left (|\mu_{s,i}-\mu_i|>\epsilon\right )\\
    \label{eq:conv_mean_exp3}
    & = \sum\limits_{i\in[K]} \sum\limits_{s = t}^\infty \P_{\bmu}\left (|\mu_{s,i}-\mu_i|>\epsilon\;,\; T_{s,i}\geq \sqrt{\frac{s}{K}} - 1\right)\nonumber\\
    &\qquad +\underbrace{\sum\limits_{i\in[K]} \sum\limits_{s = t}^\infty \P_{\bmu}\left (|\mu_{s,i}-\mu_i|>\epsilon\;,\; T_{s,i}< \sqrt{\frac{s}{K}} - 1\right)}_{=0}\\
    \label{eq:conv_mean_exp4}
    & = \sum\limits_{i\in[K]} \sum\limits_{s = t}^\infty\sum\limits_{\ell = \sqrt{s/K}-1}^\infty \P_{\bmu}(|\mu_{s,i}-\mu_i|>\epsilon\;,\; T_{s,i} = \ell)\\
    \label{eq:chernoff_bound}
    & \leq \sum\limits_{i\in[K]} \sum\limits_{s = t}^\infty\sum\limits_{\ell = \sqrt{s/K}-1}^\infty \Big ( \exp\left(-\ell {d_{\sf KL}} (\mu_i + \epsilon\|\mu_i)\right) + \exp\left(-\ell {d_{\sf KL}} (\mu_i - \epsilon\|\mu_i)\right)\Big )\\
    &\leq \underbrace{\sum\limits_{i\in[K]} \sum\limits_{s = t}^\infty \displaystyle\int_{\sqrt{s/K}-2}^\infty  \exp\left\{-\ell {d_{\sf KL}} (\mu_i + \epsilon\|\mu_i)\right\}\diff\ell}_{\triangleq A_1} \nonumber\\
    &\qquad+\underbrace{\sum\limits_{i\in[K]} \sum\limits_{s = t}^\infty \displaystyle\int_{\sqrt{s/K}-2}^\infty  \exp\left\{-\ell {d_{\sf KL}} (\mu_i - \epsilon\|\mu_i)\right\} \diff \ell}_{\triangleq A_2} \ ,
\end{align}
where~(\ref{eq:conv_mean_exp2}) is obtained by the union bound, (\ref{eq:conv_mean_exp3}) and~(\ref{eq:conv_mean_exp4}) use total probability and (\ref{eq:chernoff_bound}) is a result of using Lemma~\ref{lemma:Chernoff}. Furthermore, for $A_1$ we have
\begin{align}
    A_1&=\sum\limits_{i\in[K]} \sum\limits_{s = t}^\infty \displaystyle\int_{\sqrt{s/K}-2}^\infty  \exp\left\{-\ell {d_{\sf KL}} (\mu_i + \epsilon\|\mu_i)\right\}\diff\ell\nonumber\\
    & = \sum\limits_{i\in[K]} \frac{1}{{d_{\sf KL}} (\mu_i+\epsilon\|\mu_i)} \sum\limits_{s = t}^\infty \exp\left(-\left(\sqrt{\frac{s}{K}} - 2\right){d_{\sf KL}} (\mu_i+\epsilon\|\mu_i)\right)\\
    \label{eq:mean_conv_iter1}
    &\leq\sum\limits_{i\in[K]}\frac{\exp\left ({2{d_{\sf KL}} (\mu_i+\epsilon\|\mu_i)}\right)}{{d_{\sf KL}} (\mu_i + \epsilon\| \mu_i)}\displaystyle\int_t^\infty \exp\left(-\sqrt{\frac{s}{K}}{d_{\sf KL}} (\mu_i+\epsilon\|\mu_i)\right )\diff s \\
    &= \sum\limits_{i\in[K]} 2K\frac{\exp\left ({2{d_{\sf KL}} (\mu_i+\epsilon\|\mu_i)}\right)}{({d_{\sf KL}} (\mu_i+\epsilon\|\mu_i))^3}\left(\frac{{d_{\sf KL}} (\mu_i+\epsilon\|\mu_i)}{\sqrt{K}}\sqrt{t} + 1\right)\exp\left\{-\frac{{d_{\sf KL}} (\mu_i+\epsilon\|\mu_i)}{\sqrt{K}}\sqrt{t}\right\}\ ,
    \label{eq:explore1}
\end{align}
where~(\ref{eq:mean_conv_iter1}) is obtained by upper-bounding the summation over the index $s$ by its integration. Consequently,  
\begin{align}
    \sum\limits_{t\in\N} A_1&\leq\sum\limits_{i\in[K]} 2K\frac{\exp\left ({2{d_{\sf KL}} (\mu_i+\epsilon\|\mu_i)}\right )}{({d_{\sf KL}} (\mu_i+\epsilon\|\mu_i))^3}\displaystyle\int_0^\infty \left(\frac{{d_{\sf KL}} (\mu_i+\epsilon\|\mu_i)}{\sqrt{K}}\sqrt{t} + 1\right)\exp\left\{-\frac{{d_{\sf KL}} (\mu_i+\epsilon\|\mu_i)}{\sqrt{K}}\sqrt{t}\right\} \diff t\\
    &= 12K^2\sum\limits_{i\in[K]}\frac{\exp\left ({2{d_{\sf KL}} (\mu_i+\epsilon\|\mu_i)}\right )}{({d_{\sf KL}} (\mu_i+\epsilon\|\mu_i))^5}<+\infty\ .
    \label{eq:explore3}
\end{align}
Similarly, we can show that
\begin{align}
    \sum\limits_{t\in\N}A_2 \leq 12K^2\sum\limits_{i\in[K]}\frac{\exp\left ({2{d_{\sf KL}} (\mu_i-\epsilon\|\mu_i)}\right )}{({d_{\sf KL}} (\mu_i-\epsilon\|\mu_i))^5} < +\infty\ .
    \label{eq:explore4}
\end{align}
% \begin{align}
%     A_2 = &\sum\limits_{i\in[K]} \sum\limits_{s = t}^\infty \displaystyle\int\limits_{\sqrt{s/K}-2}^\infty  \exp\left\{-\ell {d_{\sf KL}} (\mu_i - \epsilon\|\mu_i)\right\}\diff\ell\nonumber\\
%     &\;\;\leq \sum\limits_{i\in[K]} 2K\frac{\e^{2{d_{\sf KL}} (\mu_i-\epsilon\|\mu_i)}}{({d_{\sf KL}} (\mu_i-\epsilon\|\mu_i))^3}\left(\frac{{d_{\sf KL}} (\mu_i-\epsilon\|\mu_i)}{\sqrt{K}}\sqrt{t} + 1\right)\exp\left\{-\frac{{d_{\sf KL}} (\mu_i-\epsilon\|\mu_i)}{\sqrt{K}}\sqrt{t}\right\}\ ,
%     \label{eq:explore2}
% \end{align}
% and
% \begin{align}
%     \sum\limits_{t\in\N}A_2 &\leq\displaystyle\int_0^\infty \left(\frac{{d_{\sf KL}} (\mu_i+\epsilon\|\mu_i)}{\sqrt{K}}\sqrt{t} + 1\right)\exp\left\{-\frac{{d_{\sf KL}} (\mu_i+\epsilon\|\mu_i)}{\sqrt{K}}\sqrt{t}\right\} \diff t\\
%     &= \frac{6K}{{(\sf D_{KL}}(\mu_i+\epsilon\|\mu_i))^2}\\
%     &<\infty\ .
%     \label{eq:explore4}
% \end{align}
Finally, using~(\ref{eq:explore3}) and~(\ref{eq:explore4}), we obtain
\begin{align}
    \E_{\bmu}[N_{\epsilon}^\mu] &= \sum\limits_{t\in\N} \P_{\bmu}(N_{\epsilon}^\mu>t) \leq\sum\limits_{t\in\N} (A_1 + A_2) < +\infty \ .
\end{align}
% For the next part, we note that by the estimator concentration in Lemma~\ref{lemma:estimator_B}, following the same line of arguments as~(\ref{eq:ss1})-(\ref{eq:ss2}) with probability $1$, we have
% \begin{align}
% \label{eq:ss1_B}
%     |\mu_{n,i} - \mu_i| &\leq  W_1 \sqrt{\frac{\log (\e+T_{n,i})}{1+T_{n,i}}}\\
%     &\leq  W_1 \sqrt{\frac{\log(\e + \sqrt{n/K}-1)}{\sqrt{n/K}}}\ .
%     %&\leq W_1 \sqrt{\frac{2(n/K)^{1/4}}{1+\sqrt{n/K}}}\\
%     %&\leq \frac{\epsilon}{ W_1} W_1\ ,
%     %\label{eq:ss2_B}
% \end{align}
% Finally, defining $N_{\epsilon}^\mu = {\sf poly}(1/\epsilon, W_1)$ such that
% \begin{align}
% \sqrt{\frac{\log(\e + \sqrt{N_{\epsilon}^\mu/K}-1)}{\sqrt{N_{\epsilon}^\mu/K}}}\leq \frac{\epsilon}{ W_1}\ ,
%     \label{eq:ss2_B}
% \end{align}
% we obtain that for all $n>N_{\epsilon}^\mu$ and for all $i\in[K]$,
% \begin{align}
%     |\mu_{n,i} - \mu_i| \leq \epsilon\ .
% \end{align}

% where~(\ref{eq:ss2_B}) holds for any $n>N_\epsilon$, where we have defined $N_\epsilon\triangleq {\sf poly}(W_1,1/\epsilon)$.

\subsection{Gaussian Bandits}
For the case of Gaussian bandits, 
%we will prove that under TT-SPRT, there exists $N_\epsilon = {\sf poly}(1/\epsilon, W_1)$ such that for all $n>N_\epsilon$,
% \begin{align}\label{eq:conc_mean}
%     |\mu_{n,i} - \mu_i | \leq \epsilon\quad\forall i\in[K]\ .
% \end{align}
by invoking the estimator concentration specified in Lemma~\ref{lemma:estimator} and the sufficiency of exploration established in Lemma~\ref{theorem:explore_gaussian}, for all $i\in[K]$ and $\forall n> N_1$, where $N_1$ is specified in Lemma~\ref{lemma:estimator}, we almost surely have
\begin{align}
\label{eq:ss1}
    |\mu_{n,i} - \mu_i| &\stackrel{(\ref{eq:estimator_con})}{\leq} \sigma W_1 \sqrt{\frac{\log (\e+T_{n,i})}{1+T_{n,i}}}\\
    \label{eq:v}
    &\leq \sigma W_1 \sqrt{\frac{\log(\e + \sqrt{n/K})}{1+\sqrt{n/K}}}\\
    \label{eq:vi}
    &\leq\sigma W_1 \sqrt{\frac{2(n/K)^{1/4}}{1+\sqrt{n/K}}}\ ,
\end{align}
where (\ref{eq:v}) is a result of Lemma~\ref{theorem:explore_gaussian}, noting that $n>N_1$, and monotonicity of $\frac{\log(\e+x)}{1+x}$, and (\ref{eq:vi}) holds since $\log(\e+\sqrt{n/K})\leq 2(n/K)^{1/4}$ for $n>K$. Furthermore, owing to $\frac{x^{1/4}}{(1+\sqrt{x})}$ being a decreasing function in $x$, corresponding to any realization of $W_1$ there exists $L_{\epsilon}^\mu$, which is a function of $W_1$ and an instance-dependent parameter, such that it satisfies 
%let us define $L_{\epsilon}^\mu= {\sf poly}(1/\epsilon, W_1)$ such that
\begin{align}
    \sqrt{\frac{2 (L_{\epsilon}^\mu/K)^{1/4}}{1+\sqrt{L_{\epsilon}^\mu/K}}} < \frac{\epsilon}{\sigma W_1}\ .
    \label{eq:via}
\end{align}
%Note that a choice of $L_\epsilon^\mu$ always exists, since $f(x)\triangleq (\log(\e + x))/(1+x)$ is a decreasing function in $x$. 
By defining $M_\epsilon^\mu\triangleq \max\{N_1,L_\epsilon^\mu\}$, from~(\ref{eq:vi}) and~(\ref{eq:via}), we obtain that for all $n>M_\epsilon^\mu$
\begin{align}
    |\mu_{n,i} - \mu_i|\leq \epsilon\ ,\quad \forall i\in[K]\ .
    \label{eq:ss2}
\end{align}
Since both $N_1$ and $L_\epsilon^\mu$ are a function of $W_1$ and instance dependent parameters, we have $\E_{\bmu}[N_1]<+\infty$ and $\E_{\bmu}[L_\epsilon^\mu]<+\infty$, since, by Lemma~\ref{lemma:estimator}, $\E_{\bmu}[W_1]<+\infty$. Finally, since $M_\epsilon^\mu<N_1 + L_\epsilon^\mu$, we have $\E_{\bmu}[M_\epsilon^\mu]<\E_{\bmu}[N_1]+\E_{\bmu}[L_\epsilon^\mu]<+\infty$.

\section{Proof of Lemma~\ref{theorem:proportions}}\label{proof:exp_allocation}
% In this section, we provide the proof of Lemma~\ref{theorem:proportions} for the exponential family of bandits. The proof for the Gaussian bandits is a special case, and can be obtained following the same set of steps below. 
The analysis for the convergence of the sampling proportions to the optimal allocation has two main steps. 

\noindent \textbf{Step 1. Convergence in allocation for $a^\star$}: First, we show the convergence of the selection proportions allocated to the best arm $a^\star$ to the $\beta$-optimal allocation $\beta$. For this, using Lemma~\ref{theorem:mean:convergence:exp}, we have that for all $n>N_{\Delta_{\min}/2}^\mu$, $a_n^1 = a^\star$. To show this, let us assume that $a_n^1\neq a^\star$ for some $n>N_{\Delta_{\min}/2}^\mu$. We obtain,
\begin{align}
\label{eq:conv_top2}
    %\label{eq:conv_top2} 
    \mu_{n,a_n^1} - \mu_{n,a^\star} &\leq \mu_{a_n^1} - \mu_{a^\star} + \Delta_{\min} \leq0\ ,
\end{align}
where the first inequality holds holds since $n>N_{\Delta_{\min}/2}^\mu$, and the second one follows by noting that $\mu_{a_n^1}-\mu_{a^\star} \leq -\Delta_{\min}$. Since~(\ref{eq:conv_top2}) contradicts the definition of $a_n^1$, we have that $a_n^1=a^\star$ for all $n>N_{\Delta_{\min}/2}^\mu$. Following the same line of arguments as in \cite[Lemma 12]{TTEI} and~\cite[Lemma 13]{TTEI}, the above property means that there exists a random variable $N_2^\epsilon$, which is a function of $W_1$, such that for all $n>N_2^\epsilon$, we have $|T_{n,a^\star}/n - \beta|\leq \epsilon$.

\noindent\textbf{Step 2. Convergence in allocation for any $i\in[K]\setminus\{a^\star\}$}: The other step is to show the convergence of the sampling probabilities of the other arms $i\in[K]\setminus\{a^\star\}$ to the optimal proportions $\{\omega^\star_i (\beta) : i\in[K]\setminus\{a^\star\}\}$.
%\begin{lemma}
% \label{lemma:prob_conv_bern}
% Under TT-SPRT with parameter $\beta\in(0,1)$ and for any $\epsilon>0$, there exists $N_3^\epsilon\in\N$, $\E_{\bmu}[N_3^\epsilon]<\infty$, such that for all $n>N_3^\epsilon$, and for all $i\in[K]\setminus\{a^\star\}$, we have
% \begin{align}
%     \left\lvert \frac{T_{n,i}}{n}-\omega^\star_i (\beta)\right\rvert \leq \epsilon\ .
% \end{align}
% %where $c_6>0$ is a universal constant.
% \end{lemma}
%\begin{proof}
We begin by defining the set of over-sampled arms as follows. For any $\epsilon\in\R_+$ we define
\begin{align}
    \mcO_n(\epsilon)\triangleq \Big\{ i\in[K]\setminus\{a^\star\} : T_{n,i}/n > \omega^\star_i (\beta) + \epsilon\Big\}\ .
\end{align}
Furthermore, we use the notation $\bar{\mcO}_n(\epsilon)\triangleq [K]\setminus\mcO_n(\epsilon)$ to denote the set of arms that are not over-sampled.
% and under-sampled arms,
% \begin{align}
%     \mcP_n \triangleq \Big\{ i\in[K]\setminus\{a^\star\} : T_{n,i}/n \leq \omega^\star_i (\beta)\Big \}\ .
% \end{align}
The key idea for showing the convergence of the arms $i\in[K]\setminus\{a^\star\}$ is to show that after some time, the challenger is never contained in the over-sampled set. This implies that when the means of all the arms and the proportion of the best arm have converged sufficiently, the TT-SPRT sampling strategy never samples from the over-sampled set of arms. This, in turn, leads to the convergence in allocation of the arms $i\in[K]\setminus\{a^\star\}$. This idea is formalized in the following lemma, which has been proved in~\cite{TTEI}.
\begin{lemma}[\cite{TTEI}]
\label{lemma:2_np}
For any $\epsilon>0$, for any top-two sampling strategy, assume that the following conditions are satisfied:
\begin{enumerate}
    \item There exists a stochastic time $N_2^\epsilon$, such that for all $n>N_2^\epsilon$, we have convergence in allocation for the best arm, i.e.,
    \begin{align}
        \left\lvert \frac{T_{n,a^\star}}{n} - \beta\right\rvert\;\leq\;\epsilon\ ,\quad \forall n>N_2^\epsilon\ .
    \end{align}
    \item There exists $N_3^\epsilon$ such that for all $n>N_3^\epsilon$, the challenger $a_n^2\notin\mcO_n(\epsilon/2)$.
\end{enumerate}
Then, there exists $N_4^\epsilon$ such that $\E_{\bmu}[N_4^\epsilon]<+\infty$, and $\mcO_n(\epsilon)=\emptyset$ for all $n>N_4^\epsilon$. Furthermore, for all $n>N_4^{\epsilon/K}$, we have convergence in allocation of every arm $i\in[K]\setminus\{a^\star\}$, i.e.,
\begin{align}
        \left\lvert \frac{T_{n,a^\star}}{n} - \beta\right\rvert\;\leq\;\epsilon\ ,\quad \forall n>N_4^{\epsilon/K},\;\forall i\in[K]\ .
    \end{align}
\end{lemma}
% \begin{lemma}[Lemma 14,~\cite{TTEI}]
% \label{lemma:2_np}
% If any top-two allocation rule with parameter $\beta\in(0,1)$ satisfies that $|T_{n,a^\star}/n - \beta| \leq \epsilon$, and satisfies the condition that for any $\epsilon>0$, there exists $M_\epsilon$ such that for all $n>M_\epsilon$, $a_n^2\notin \mcO_n(\epsilon/2)$, %, where we have defined
% % \begin{align}
% %     \mcO_n(\epsilon)\triangleq \Big\{ i\in[K] : T_{n,i}/n > \omega^\star_i (\beta) + \epsilon\Big\}\ ,
% % \end{align}
% then, there exists $N_3^\epsilon = {\sf poly}(1/\epsilon,W_2)$ such that for all $i\in[K]\setminus\{a^\star\}$ and for all $n>N_3^\epsilon$, it holds that
% \begin{align}
%     \bigg\lvert \frac{T_{n,i}}{n}-\omega^\star_i (\beta)\bigg\rvert \leq \epsilon\ .
% \end{align}
% \end{lemma}
Leveraging Lemma~\ref{lemma:2_np}, the proof is concluded by setting $N_\epsilon^\omega\triangleq N_4^{\epsilon/K}$. We have already shown that condition $1$ in Lemma~\ref{lemma:2_np} is satisfied by TT-SPRT for all $n>N_2^\epsilon$. Next, we will show that condition $2$ is also satisfied. Specifically, we will show that there exists $N_3^\epsilon$ such that for all $n>N_3^\epsilon$, for any $i\in\mcO_n(\epsilon/2)$ and any $j\in\bar{\mcO}_n(0)$ we have
\begin{align}
\label{eq:conv_alloc_new1}
    \Lambda_n(a_n^1,i)>\Lambda_n(a_n^1,j)\ .
\end{align}
%there exists a time instant $N_3^\epsilon$ such that for all $n>N_3^\epsilon$, $a_n^2\notin\mcO_n(\epsilon/2)$. 
% Thus, it is sufficient to show that for $n>N_3^\epsilon$, for any $i\in\mcO_n(\epsilon/2)$,
% \begin{align}
%     \Lambda_n(a_n^1,i)>\Lambda_n(a_n^1,a_n^2)\ .
% \end{align}
%We begin by defining the set of under-sampled arms,
% \begin{align}
%     \mcP_n \triangleq \Big\{ i\in[K] : T_{n,i}/n \leq \omega^\star_i (\beta)\Big \}\ .
% \end{align}
% We will show a more general condition. Specifically, we will show that there exists $N_3^\epsilon$ such that for all $n>N_3^\epsilon$, for any $i\in\mcO_n(\epsilon/2)$ and any $j\in\bar{\mcO}_n(0)$,
% \begin{align}
%     \Lambda_n(a_n^1,i)>\Lambda_n(a_n^1,j)\ .
% \end{align}
By the definition of $a_n^2$, this implies that $a_n^2\notin \mcO_n(\epsilon/2)$. Next, we will show that~(\ref{eq:conv_alloc_new1}) holds.
%By [\cite{TTEI}, Lemma 11], we have that $\mcO_n(\epsilon/2)\neq\emptyset$ implies that $\mcP_n\neq\emptyset$. We will show that there exists $M_\epsilon$ such that for all $n>M_\epsilon$, for any $i\in\mcO_n(\epsilon/2)$ and any $j\in\mcP_n$, $\Lambda_n(a_n^1,i)>\Lambda_n(a_n^1,j)$.
% We invoke Lemma~\ref{lemma:2_np}, which provides a sufficient condition for the convergence of the sampling allocations to the optimal ones. Specifically, it is sufficient to prove that there exists $M_\epsilon\in\N$, such that for any $n>M_\epsilon$, $a^2_n\notin\mcO^{\epsilon/2}_n$, where we recall that the definition of the set of over-sampled arms is given by
% \begin{align}
%     \mcO_n(\epsilon) \triangleq \left\{i\in[K]\setminus\{a^\star\} : \frac{T_{n,i}}{n} > \omega^\star_i (\beta) + \epsilon\right\}\ .
% \end{align}
For any $i\in[K]\setminus\{a^\star\}$, let us define
\begin{align}
    C_i(\beta,\omega^\star_i(\beta))\triangleq \beta{d_{\sf KL}} (\mu_{a^\star}\|\mu_{a^\star,i}) + \omega^\star_i (\beta) {d_{\sf KL}} (\mu_i\|\mu_{a^\star,i})\ ,
\end{align}
where $\mu_{a^\star,i}\triangleq (\beta\mu_{a^\star} + \omega^\star_i (\beta)\mu_i)/(\beta+\omega^\star_i (\beta))$. Next, we state a proposition characterizing the optimal allocation. 
\begin{proposition}[\cite{pmlr-v108-shang20a},~Proposition $1$]
\label{prop:optimal_alloc}
There exists a unique optimal allocation $\bomega^\star(\beta)$ to the optimization problem in~(\ref{eq:PC}), such that for any $i,j\in[K]\setminus\{a^\star\}$, $C_i(\beta,\omega^\star_i (\beta))=C_j(\beta,\omega^\star_j (\beta))$. Furthermore,
\begin{align}
    C_i(\beta,\omega^\star_i(\beta)) = \min\limits_{x\in[\mu_i,\mu_{a^\star}]} \beta{d_{\sf KL}} (\mu_{a^\star}\|x) + \omega^\star_i (\beta){d_{\sf KL}} (\mu_i\|x)\ .
\end{align}
\end{proposition}
% Let us recall the definition of the set of under-sampled arms
% \begin{align}
%     \mcP_n \triangleq \left\{i\in[K]\setminus\{a^\star\} : \frac{T_{n,i}}{n} < \omega^\star_i (\beta)\right\}\ .
% \end{align}
%In order to prove that for $n>M_\epsilon$, $a_n^2\notin\mcO_n(\epsilon/2)$, it is sufficient to show that for any $i\in\mcO_n(\epsilon/2)$ and $j\in\mcP_n$, we have $\Lambda_n(a_n^1,i)>\Lambda_n(a_n^1,j)$. 
\noindent Let us set $\zeta=\epsilon/2$. Leveraging Lemma~\ref{theorem:mean:convergence:exp}, we obtain that for all $n>N_{\zeta^2}^\mu$, $|\mu_{n,i}-\mu_i|<\zeta^2$ for all $i\in[K]$. Next, for all $n>N_2^{\zeta^2}$, we have $|T_{n,a^\star}/n-\beta| < \zeta^2$. Let us define $N_3^\zeta\triangleq \max\{N_{\zeta^2}^\mu, N_2^{\zeta^2}, N_{\Delta_{\min}/2}^\mu\}$. Thus, for all $n>N_3^\zeta$, for any $i\in\mcO_n(\zeta)$ and for any $j\in\bar{\mcO}_n(0)$:
\begin{align}
\label{eq:a0}
    &\frac{1}{n}\left( \Lambda_n(a_n^1,i) - \Lambda_n(a_n^1,j)\right)
    %\nonumber\\&
    = \frac{1}{n}\left(\Lambda_n(a^\star,i) - \Lambda_n(a^\star,j)\right)\\
    & = \underbrace{\frac{1}{n}\Lambda_n(a^\star,i) - C_i(\beta,\omega^\star_i(\beta) )}_{\triangleq A_3} + \underbrace{C_j(\beta,\omega^\star_j(\beta)) - \frac{1}{n}\Lambda_n(a_n^1,j)}_{\triangleq A_4}\ ,
    \label{eq:a}
\end{align}
where~(\ref{eq:a0}) holds since $a_n^1 = a^\star$ for all $n>N_3^\zeta$, and~(\ref{eq:a}) is obtained by leveraging Proposition~\ref{prop:optimal_alloc}. Expanding $A_3$, we obtain
\begin{align}
    A_3&=\frac{1}{n}\Lambda_n(a^\star,i) - C_i(\beta,\omega^\star_i(\beta))\nonumber\\
    %&\;\;= \frac{T_{n,a^\star}}{n}{d_{\sf KL}} (\mu_{n,i}\|\mu_{n,a^\star,i}) + \frac{T_{n,i}}{n}{d_{\sf KL}} (\mu_{n,i}\|\mu_{n,a^\star,i}) \nonumber\\&\qquad\qquad-C_i(\beta,\omega^\star_i(\beta))\\
    &\geq (\beta-\zeta^2){d_{\sf KL}} (\mu_{n,a^\star}\|\mu_{n,a^\star,i}) - \beta{d_{\sf KL}} (\mu_{a^\star}\|\mu_{a^\star,i})\nonumber\\
    &\quad + (\omega^\star_i (\beta)+\zeta){d_{\sf KL}} (\mu_{n,i}\|\mu_{n,a^\star,i}) - \omega^\star_i (\beta) {d_{\sf KL}} (\mu_i\|\mu_{a^\star,i})   \ ,
    \label{eq:num1}
\end{align}
where~(\ref{eq:num1}) holds due to the definition of $N_3^\zeta$, along with the fact that $i\in\mcO_n(\zeta)$ and $j\in\bar{\mcO}_n(0)$.
% Note that there exists a universal constant $c_1>0$, such that
% \begin{align}
%     \mu_{n,a^\star,i}\geq \mu_{a^\star,i} - c_1\zeta
% \end{align}
%for all $n>N_3^\zeta$. Furthermore, for all $n>N_3^\zeta$, $\mu_{n,i}\leq \mu_i+\zeta^2$ for all $i\in[K]$. 
Now,~(\ref{eq:num1}) can be lower-bounded as
\begin{align}
    &\quad(\beta-\zeta^2){d_{\sf KL}} (\mu_{n,a^\star}\|\mu_{n,a^\star,i}) - \beta{d_{\sf KL}} (\mu_{a^\star}\|\mu_{a^\star,i})\nonumber\\
    &\quad + (\omega^\star_i (\beta)+\zeta){d_{\sf KL}} (\mu_{n,i}\|\mu_{n,a^\star,i}) - \omega^\star_i (\beta) {d_{\sf KL}} (\mu_i\|\mu_{a^\star,i})\\
    % &\frac{1}{n}\Lambda_n(a_n^1,i) - C_i(\beta,\omega^\star_i(\beta)) \nonumber\\
    \label{eq:KLconv1}
    &\quad\geq (\beta-\zeta^2){d_{\sf KL}} (\mu_{a^\star} - \zeta^2 \| \mu_{n,a^\star,i}) - \beta{d_{\sf KL}} (\mu_{a^\star}\| \mu_{a^\star,i})\nonumber\\
    &\qquad + (\omega^\star_i (\beta) +\zeta) {d_{\sf KL}} (\mu_i+\zeta^2\|\mu_{n,a^\star,i}) - \omega^\star_i (\beta) {d_{\sf KL}} (\mu_i\|\mu_{a^\star,i})\\
    &\quad\geq (\beta-\zeta^2){d_{\sf KL}} (\mu_{a^\star} - \zeta^2 \| \mu_{n,a^\star,i}) - \beta{d_{\sf KL}} (\mu_{a^\star}\| \mu_{n,a^\star,i})\nonumber\\
    \label{eq:num12}
    &\qquad + (\omega^\star_i (\beta) +\zeta) {d_{\sf KL}} (\mu_i+\zeta^2\|\mu_{n,a^\star,i}) - \omega^\star_i (\beta) {d_{\sf KL}} (\mu_i\|\mu_{n,a^\star,i})\\
    &\quad = (\beta-\zeta^2)\left[{d_{\sf KL}} (\mu_{a^\star}\|\mu_{n,a^\star,i})+O(\zeta^2)\right] - \beta{d_{\sf KL}} (\mu_{a^\star}\|\mu_{n,a^\star,i})\nonumber\\
    \label{eq:num2}
    &\qquad +(\omega^\star_i (\beta)+\zeta)\left[{d_{\sf KL}} (\mu_i\|\mu_{n,a^\star,i}) + O(\zeta^2)\right] - \omega^\star_i (\beta){d_{\sf KL}} (\mu_i\|\mu_{n,a^\star,i})\\
    &\quad = \zeta{d_{\sf KL}} (\mu_i\|\mu_{n,a^\star,i}) + O(\zeta^2)\ ,
%    &\quad = \zeta{d_{\sf KL}} (\mu_i\|\mu_{a^\star,i}) + O(\zeta^2)\ ,
    \label{eq:b}
\end{align}
%where~(\ref{eq:num2}) is obtained by leveraging Lemma~\textcolor{red}{placeholder}. 
where~(\ref{eq:KLconv1}) follows from the facts that $\mu_{n,a^\star,i}\in[\mu_{n,i},\mu_{n,a^\star}]$ and $|\mu_{n,i}-\mu_i|\leq\zeta^2$ for every $n>N_3^\zeta$, (\ref{eq:num12}) follows from Proposition~\ref{prop:optimal_alloc} by noting that $\beta{d_{\sf KL}} (\mu_{a^\star}\|\mu_{n,a^\star,i}) + \omega^\star_i (\beta){\sf  D_{KL}}(\mu_i\|\mu_{n,a^\star,i})\geq\beta{d_{\sf KL}} (\mu_{a^\star}\|\mu_{a^\star,i}) + \omega^\star_i (\beta){\sf  D_{KL}}(\mu_i\|\mu_{a^\star,i})$, and~(\ref{eq:num2}) is a result of Lemma~\ref{lemma:KL_conv}. Similarly, expanding $A_4$, we obtain
%vspace{0.3in}
\begin{align}
    A_4&=\frac{1}{n}\Lambda_n(a_n^1,j) - C_j(\beta,\psi^\beta_j)\nonumber\\
    & = \frac{T_{n,a^\star}}{n}{d_{\sf KL}} (\mu_{n,a^\star}\|\mu_{n,a^\star,j}) + \frac{T_{n,j}}{n}{d_{\sf KL}} (\mu_{n,j}\|\mu_{n,a^\star},j)
    - C_j(\beta,\omega^\star_j(\beta))\\
    \label{eq:num3_0}
    &\leq \frac{T_{n,a^\star}}{n}{d_{\sf KL}} (\mu_{n,a^\star}\|\mu_{a^\star,j}) + \frac{T_{n,j}}{n}{d_{\sf KL}} (\mu_{n,j}\|\mu_{a^\star},j) - C_j(\beta,\omega^\star_j(\beta))\\
    \label{eq:num3_1}
    &\leq (\beta+\zeta^2){d_{\sf KL}} (\mu_{n,a^\star}\|\mu_{a^\star,j}) - \beta{d_{\sf KL}} (\mu_a^\star\|\mu_{a^\star,j})\nonumber\\
% \end{align}
% \begin{align}
    &\qquad+\omega^\star_j (\beta){d_{\sf KL}} (\mu_{n,j}\|\mu_{a^\star,j}) - \omega^\star_j(\beta){d_{\sf KL}} (\mu_j\|\mu_{a^\star,j})\\
    &\leq (\beta+\zeta^2){d_{\sf KL}} (\mu_{a^\star}+\zeta^2\|\mu_{a^\star,j}) - \beta{d_{\sf KL}} (\mu_{a^\star}\|\mu_{a^\star,j})\nonumber\\
    \label{eq:num3}
    &\qquad +\omega^\star_j(\beta) {d_{\sf KL}} (\mu_j-\zeta^2\|\mu_{a^\star,j}) - \omega^\star_j(\beta){d_{\sf KL}} (\mu_j\|\mu_{a^\star,j})\\
    &\leq (\beta+\zeta^2)\left[{d_{\sf KL}} (\mu_{a^\star}\|\mu_{a^\star,j})+O(\zeta^2)\right] - \beta{d_{\sf KL}} (\mu_{a^\star}\|\mu_{a^\star,j})\nonumber\\
    \label{eq:num4}
    &\qquad +\omega^\star_j(\beta)\left[{d_{\sf KL}} (\mu_j\|\mu_{a^\star,j}) + O(\zeta^2)\right] - \omega^\star_j(\beta){d_{\sf KL}} (\mu_j\|\mu_{a^\star,j})\\
    &= O(\zeta^2)\ ,
    \label{eq:c}
\end{align}
where~(\ref{eq:num3_0}) holds due to the fact that~\cite{pmlr-v49-garivier16a}
\begin{align}
    \mu_{n,a^\star,i}\triangleq \argmin\limits_{x\in[\mu_{n,i},\mu_{n,a^\star}]}\;\frac{T_{n,a^\star}}{n}{d_{\sf KL}} (\mu_{n,a^\star}\|x) + \frac{T_{n,j}}{n}{d_{\sf KL}} (\mu_{n,j}\|x)\ , 
\end{align}
(\ref{eq:num3_1})-(\ref{eq:num3}) holds due to the fact that $n>N_3^\zeta$, and~(\ref{eq:num4}) is obtained by leveraging Lemma~\ref{lemma:KL_conv}.
%noting that $\mu_{n,a^\star,j}\leq \mu_{a^\star,j}+\zeta^2$ and $\mu_{n,j}\geq \mu_j-\zeta^2$ for all $j\in[K]$.
%~(\ref{eq:num4}) is a result of Lemma~\textcolor{red}{placeholder}. 
Finally, using~(\ref{eq:b}) and~(\ref{eq:c}),~(\ref{eq:a}) can be rewritten as
\begin{align}
    \frac{1}{n}\left(\Lambda_n(a_n^1,i)-\Lambda_n(a_n^1,j)\right)
    %\nonumber\\&\;\;
    &\geq \zeta{d_{\sf KL}} (\mu_i\|\mu_{n,a^\star,i}) +O(\zeta^2)\;> 0\ .
    \label{eq:d}
\end{align}
%\AM{The logic for~(\ref{eq:num3}) and~(\ref{eq:num3_1}) are the same as that of $A_3$, i.e., since $n>N_3^\zeta$. Should I mention that here as well?} 
%where~(\ref{eq:d}) follows by choosing $\zeta<c_6$, where $c_6 > 0$ is a universal constant.
%\end{proof}
%\vspace{-0.4in}
\section{Proof of Theorem~\ref{theorem:SC_exp} and Theorem~\ref{theorem:SC_Gaussian}}
\label{proof:SC_exp}
	In this section, we start by providing the proof of Theorem~\ref{theorem:SC_exp} for the case of the exponential family of bandits. We begin by defining the function $\mcC_{\rm exp}$ used in characterizing the stopping threshold in Theorem~\ref{theorem:PAC_bern}. For this purpose, for any $u\geq 1$, let us define the function $h(u)\triangleq u - \log u$. Furthermore, let us define the function $\tilde h_z(x) : \R_+\mapsto \R_+$ for any $z\in[1,\e]$ as
\begin{align}\label{eq:h_z}
	h_z(x) \triangleq \left\{
	\begin{array}{ll}
		\e^{\frac{1}{h^{-1}(x)}}h^{-1}(x)\ , & \mbox{if} \;\; x\geq h(1/\log z)\ ,\\
		z(x - \log\log z)\ ,  & \mbox{otherwise} \\
	\end{array}\right. \ .
\end{align}
The function $\mcC_{\rm exp} : \R_+\mapsto\R_+$ is defined as
\begin{align}
	\mcC_{\rm exp}(x)\;\triangleq\; 2\tilde h_{3/2}\left ( \frac{h^{-1}(1+x) + \log(2\zeta(2))}{2}\right )\ ,
\end{align}
where $\zeta(s)\triangleq\sum\limits_{n=1}^\infty n^{-s}$.
Let us define the stochastic time $T_\beta^\epsilon$, which marks the convergence of the arm means and the sampling proportions,
\begin{align}
	\label{eq:T_beta_epsilon}
	T_\beta^\epsilon\triangleq &\inf\Big\{ n_0\in\N : |\mu_{n,i}-\mu_i|\leq \epsilon,
	\nonumber\\&
	|T_{n,i}/n - \omega^\star_i (\beta)|\leq \epsilon\;\forall i\in[K],\;\forall n>n_0\Big \}\ .
\end{align}
Leveraging Lemma~\ref{theorem:mean:convergence:exp}, we have $\E_{\bmu}[N_{\epsilon}^\mu]<+\infty$, and $|\mu_{n,i}-\mu_i|\leq \epsilon$ for all $n>N_\epsilon^\mu$ and for every $i\in[K]$. Furthermore, in Lemma~\ref{theorem:proportions} we showed that for any $\epsilon>0,$ there exists $N^\omega_\epsilon$, $\E_{\bmu}[N^\omega_\epsilon]<\infty$ and for all $n>N_\epsilon^\omega$ and $i\in[K]$ we have $|T_{n,i}/n - \omega_i^\star(\beta)|<\epsilon$. Furthermore, by the definition of $T_\beta^\epsilon$ in~(\ref{eq:T_beta_epsilon}), $T_\beta^\epsilon = \max\{N^\mu_\epsilon,N^\omega_\epsilon\}$, and thus $\E_{\bmu}[T_\beta^\epsilon]<+\infty$. Finally, the theorem is proved by leveraging Lemma~\ref{Lemma:SC_T} stated below.

\begin{lemma}
	\label{Lemma:SC_T}
	For the combination of any arm selection rule that satisfies $\E_{\bmu}[T_\beta^\epsilon]<+\infty$ for any $\epsilon>0$, and the stopping rule specified in~(\ref{eq:stop}), we have
	\begin{align}
		\lim\limits_{\delta\rightarrow 0}\frac{\E_{\bmu}[\tau]}{\log(1/\delta)}\leq \frac{1}{\Gamma_{\bmu}(\beta)}\ .
	\end{align}
\end{lemma}
\begin{proof}
Note that by the definition of $T_\epsilon^\beta$, using the same argument as in~(\ref{eq:conv_top2}), we have $a_n^1=a^\star$ for all $n\geq T_\beta^{\Delta_{\min}/2}$. Thus, for all $n\geq T_\beta^{\Delta_{\min}/2}$, we have
\begin{align}
    \frac{1}{n}\Lambda_n(a_n^1,a_n^2) \stackrel{(\ref{eq:exp_LLR})}{=} \min\limits_{j\in[K]\setminus\{a^\star\}}\;&\frac{T_{n,a^\star}}{n}{d_{\sf KL}} (\mu_{n,a^\star}\|\mu_{n,a^\star,j})
    %\nonumber\\&\; 
    +\frac{T_{n,j}}{n}{d_{\sf KL}} (\mu_{n,j}\|\mu_{n,a^\star,j})\ .
\end{align}
Furthermore, recall the definition of $\Gamma_{\bmu}(\beta)$, 
\begin{align}
    \Gamma_{\bmu}(\beta)\triangleq \min\limits_{i\in[K]\setminus a^\star}\; &\beta{d_{\sf KL}} (\mu_{a^\star}\|\mu_{a^\star,i})+\omega^\star_i (\beta) {d_{\sf KL}} (\mu_i\|\mu_{a^\star,i})\ ,
\end{align}
based on which, given $\epsilon>0$, there exist $\epsilon^\prime\in(0,\Delta_{\min}/2]$ and $T^\epsilon\triangleq T_{\beta}^{\epsilon^\prime}$, such that for all we have $n>T^\epsilon$, $|\mu_{n,i}-\mu_i|\leq \epsilon^\prime$, $|T_{n,i}/n-\omega^\star_i (\beta)|\leq \epsilon^\prime$, and $|\Lambda(a_n^1,a_n^2)/n-\Gamma_{\bmu}(\beta)|<\epsilon$~\cite{pmlr-v49-garivier16a,TTEI}. Furthermore, $\E_{\bmu}[T^\epsilon]=\E_{\bmu}[T_\beta^{\epsilon^\prime}]<+\infty$.
{Furthermore, for all $t>3\times 10^7$, we have
\begin{align}
    6\log\left ( \log\frac{t}{2} + 1\right)\;\leq\; \log t\ .
\end{align}
Let us define $T_0^\epsilon\triangleq \max\{T^\epsilon, 3\times 10^7\}$. Next, expanding the time instant right before stopping, we have
\begin{align}
    \tau - 1 &= (\tau-1)\mathds{1}_{\{\tau-1\leq T_0^\epsilon\}} + (\tau-1)\mathds{1}_{\{\tau-1>T_0^\epsilon\}}\\
    &\leq T_0^\epsilon + (\tau-1)\mathds{1}_{\{\tau-1>T_0^\epsilon\}}\ .
    \label{eq:s1}
\end{align}
Now, due to the choice of our stopping rule in~(23) along with the choice of threshold defined in Theorem~\ref{theorem:PAC_bern}, at $(\tau-1)$, we have
\begin{align}
    &\Lambda_{\tau-1}(a^{\tau-1}_1,a^{\tau-1}_2)\nonumber\\ &\;\;\leq \log\frac{K-1}{\delta} + 8\log\left ( 1 + \frac{1}{2}\log\frac{K-1}{\delta} + \sqrt{\log\frac{K-1}{\delta}}\right ) + 6\log\left(\log\frac{\tau-1}{2} + 1\right)\ .
    \label{eq:s2}
\end{align}
Furthermore, when $\tau-1>T_0^\epsilon$, 
\begin{align}
    \Lambda_{\tau-1}(a^{\tau-1}_1,a^{\tau-1}_2)\geq (\tau-1)(\Gamma_{\bmu}(\beta)-\epsilon)\ .
    \label{eq:s3}
\end{align}
Combining~(\ref{eq:s2}) and~(\ref{eq:s3}), when $\tau-1>T_0^\epsilon$, we have 
\begin{align}
    &(\tau-1)(\Gamma_{\bmu}(\beta)-\epsilon) \nonumber \\&\;\;\leq \log\frac{K-1}{\delta} + 8\log\left ( 1 + \frac{1}{2}\log\frac{K-1}{\delta} + \sqrt{\log\frac{K-1}{\delta}}\right ) + 6\log\left(\log\frac{\tau-1}{2} + 1\right)\\
    &\;\;\leq \log\frac{K-1}{\delta} + 8\log\left ( 1 + \frac{1}{2}\log\frac{K-1}{\delta} + \sqrt{\log\frac{K-1}{\delta}}\right ) + \log(\tau - 1)\ . 
\end{align}
Furthermore, note that $f(x)\triangleq x - \frac{1}{D_1}\log D_2 x$ is a monotonically increasing function in $x$ for $D_1, D_2\in\mathbb{R}_+$. Thus, there exists $x_{\max}$ such that for all $x\geq x_{\max}$, $f(x)\geq 0$. Next, we will find an $\bar x$ such that $f(\bar x)\geq 0$. In our case, we have $D_1 = (\Gamma_{\bmu}(\beta)-\epsilon)$, and $D_2 = ((K-1)g(\delta))/\delta$, where we have defined
\begin{align}
    g(\delta)\;\triangleq\; \left (1 + \frac{1}{2} \log\frac{K-1}{\delta} + \sqrt{\log\frac{K-1}{\delta}}\right )^8\ .
\end{align}
We leverage \cite[Lemma 18]{pmlr-v49-garivier16a}, which yields that
\begin{align}
    \tau-1\leq \frac{1}{\Gamma_{\bmu}(\beta) - \epsilon} \left[\log\left\{\frac{g(\delta)(K-1)\e}{(\Gamma_{\bmu}(\beta)-\epsilon)\delta}\right\}+\log\log\left\{\frac{g(\delta)(K-1)}{(\Gamma_{\bmu}(\beta) - \epsilon)\delta}\right\}\right].
    \label{eq:s4}
\end{align}
Thus, combining~(\ref{eq:s1}) and~(\ref{eq:s4}), we obtain
\begin{align}
    \tau\leq T_0^\epsilon + \frac{1}{\Gamma_{\bmu}(\beta) - \epsilon} &\Bigg[\log\left\{\frac{g(\delta)(K-1)\e}{(\Gamma_{\bmu}(\beta)-\epsilon)\delta}\right\}
    %\nonumber\\&\qquad
    +\log\log\left\{\frac{g(\delta)(K-1)}{(\Gamma_{\bmu}(\beta) - \epsilon)\delta}\right\}\Bigg] + 1\ .
    \label{eq:s5}
\end{align}
Finally, taking the expectation on both sides of~(\ref{eq:s5}), dividing by $\log(1/\delta)$, and taking the limit of $\delta\rightarrow 0$, we obtain
\begin{align}
    \lim\limits_{\delta\rightarrow 0} \frac{\E_{\bmu}[\tau]}{\log(1/\delta)}\leq \frac{1}{\Gamma_{\bmu}(\beta)-\epsilon}\ , 
    \label{eq:star1}
\end{align}
where~\eqref{eq:star1} is obtained by noting that $\lim_{\delta\rightarrow 0} \log g(\delta)/\log(1/\delta) = 0$, and we have used the fact that $\E_{\bmu}[T_0^\epsilon]<+\infty$. The proof for the exponential family of bandits is completed by taking an infimum with respect to $\epsilon$ in~(\ref{eq:star1}).}

The proof for the Gaussian setting involves a counterpart of Lemma~\ref{Lemma:SC_T}, which has been proved in~\cite[Lemma 1]{pmlr-v108-shang20a}. Specifically, \cite[Lemma 1]{pmlr-v108-shang20a} states that for Gaussian bandits, the combination of any arm selection rule that satisfies $\E_{\bmu}[T_{\beta}^\epsilon]<+\infty$, where we have defined $T_{\beta}^\epsilon$ in~(\ref{eq:T_beta_epsilon}), along with the stopping rule specified in~(\ref{eq:stop_G}), satisfies
\begin{align}
    \limsup\limits_{\delta\rightarrow 0}\; \frac{\E_{\bmu}[\tau]}{\log(1/\delta)}\;\leq\;\frac{1}{\Gamma_{\bmu}(\beta)}\ .
\end{align}
The proof is concluded by leveraging Lemma~\ref{theorem:mean:convergence:exp} and Lemma~\ref{theorem:proportions}, ensuring  $\E_{\bmu}[T_{\beta}^\epsilon]<+\infty$.
\end{proof}

\section{Proof of Theorem~\ref{theorem:TTTS_samples}}\label{appendix:D}
% We begin by restating the definition of $T_{\sf TTTS}^n$, which is given by
% \begin{align}
%     T_{{\sf TTTS}}^n\triangleq\inf\{s\in\N : \exists i\in[K], \theta_{s,i}^n>\theta_{s,b^1_s}^n\}\ .
% \end{align}
Let us define
\begin{align}
\label{eq:T_i^n}
    T_{i}^n\triangleq\inf\{s\in\N : \theta_{s,i}^n>\theta_{s,b^1_s}^n\}\ ,
\end{align}
based on which $T^n_{\sf TTTS} = \min_{i\in[K]\setminus\{b^1_n\}} T^n_i$. Note that using~\cite[Lemma 12]{pmlr-v108-shang20a}, there exists a random variable $M_0$, such that for all $n>M_0$, we have $b^1_n=a^\star$. Similarly, based on~\cite[Section C.2]{pmlr-v108-shang20a} there exists a random variable $M_{\Delta_{\min/4}}$ such that for all $n>M_{\Delta_{\min/4}}$, we almost surely have
\begin{align}\label{eq:appB_6}
    |\mu_{n,i}-\mu_i|\leq \frac{\Delta_{\min}}{4}\ ,\quad\forall i\in[K]\ .
\end{align}
Defining $N_0\triangleq \max(M_0,M_{\Delta_{\min/4}})$, for all $n>N_0$, $s\in\N$, and for any $i\in[K]\setminus\{a^\star\}$ we have
\begin{align}
    \label{eq:appB_1}
    \P\Big( \theta_{s,i}^n>\theta_{s,a^\star}^n\Big )
    %\nonumber\\&\quad
    &\leq \frac{1}{2}\exp\bigg(-\frac{(\mu_{n,a^\star}-\mu_{n,i})^2}{2\sigma^2}\left (\frac{1}{T_{n,i}}+\frac{1}{T_{n,a^\star}}\right )^{-1}\bigg )\\
    \label{eq:appB_5}
    &\leq \frac{1}{2}\exp\bigg (-\frac{(\mu_{a^\star}-\mu_i-\Delta_{\min}/2)^2}{2\sigma^2}\left (\frac{1}{T_{n,i}}+\frac{1}{T_{n,a^\star}}\right )^{-1}\bigg )\\
    \label{eq:prob_bound}
    &\leq \frac{1}{2}\exp\bigg (-\sqrt{\frac{n}{K}}\frac{(\mu_{a^\star}-\mu_i-\Delta_{\min}/2)^2}{4\sigma^2}\bigg )\ ,
\end{align}
where (\ref{eq:appB_1}) is obtained by leveraging~\cite[Lemma 1]{TTEI} while noting that for any two arms $i,j\in[K]$ and $n\in\N$, $(\theta_{s,i}^n - \theta_{s,j}^n)\sim\mcN(\mu_{n,i}-\mu_{n,j},\sigma_n^{i,j})$, where we have defined $\sigma_n^{i,j} \triangleq \sqrt{\sigma^2(1/T_{n,i} + 1/T_{n,j})}$, (\ref{eq:appB_5}) is a result of (\ref{eq:appB_6}), and (\ref{eq:prob_bound}) is a result of~\cite[Lemma 5]{pmlr-v108-shang20a}, which states that under TTTS, there exists a random variable $M_1$ such that for all $n>M_1$, we have $T_{n,i}\geq \sqrt{n/K}$ for every $i\in[K]$. Furthermore, using~\cite[Lemma 1]{TTEI}, we also obtain that
\begin{align}
\label{eq:167}
    &\P\Big( \theta_{s,i}^n>\theta_{s,a^\star}^n\Big )\geq \frac{1}{\sqrt{2\pi}}\exp\bigg (-\frac{(\sigma_n^{a^\star,i} + \mu_{n,a^\star} - \mu_{n,i})^2}{2\sigma^2}\left (\frac{1}{T_{n,i}}+\frac{1}{T_{n,a^\star}}\right )^{-1}\bigg )\ .
\end{align}
Following the same line of arguments as in (\ref{eq:appB_1})-(\ref{eq:prob_bound}), and noting that $T_{n,i}\leq n$ for every $i\in[K]$, from~(\ref{eq:167}) we obtain
\begin{align}\label{eq:appB_8}
    \P\Big( \theta_{s,i}^n>\theta_{s,a^\star}^n\Big )&\geq \frac{1}{\sqrt{2\pi \e}}\exp\Bigg (- \sqrt{n}\left (\mu_{a^\star}-\mu_i+\frac{\Delta_{\min}}{2}\right )\nonumber\\&\qquad\times\bigg(\sqrt{n}\; \frac{ \mu_{a^\star}-\mu_i+\Delta_{\min}/2}{4\sigma^2}+\frac{1}{\sqrt{2}\sigma}\bigg) \Bigg )\ .
\end{align}
Next, note that for any $n>N_0$, we have
%the probability that at least $s$ samples are required for any arm $i\in[K]$ to generate a realization that exceeds the realization of the best arm $a^\star$ is given by
\begin{align}
\label{eq:appB_2_prev0}
    \P(T^n_i > s) &= 1 - \P(T^n_i\leq s)\\
    %& = 1 - \sum\limits_{\ell=1}^s \P(T_i^n = \ell)\\
    \label{eq:appB_2_prev}
    &= 1 - \sum\limits_{\ell=1}^s \P\big( \forall m<\ell\;,\; \theta_{m,i}^n\leq\theta_{m,a^\star}^n\;,\; \theta_{\ell,i}^n > \theta_{\ell,a^\star}^n \big)\\
    \label{eq:aapB_2}
    &= 1 - \sum\limits_{\ell=1}^s \P(\theta_{\ell,i}^n > \theta_{\ell,a^\star}^n)\displaystyle\prod\limits_{m=1}^\ell \P(\theta_{m,i}^n\leq\theta_{m,a^\star}^n)\ ,
\end{align}
where~(\ref{eq:appB_2_prev}) follows from the definition of $T_i^n$ in~(\ref{eq:T_i^n}), and~(\ref{eq:aapB_2}) holds since the samples are drawn independently from the posterior $\Pi_n$. For any $m\in\N$, let us define $p_{n,i} \triangleq  \P(\theta_{m,i}^n\leq\theta_{m,a^\star}^n)$. Subsequently,~(\ref{eq:aapB_2}) can be simplified as follows
\begin{align}
    \P(T^n_i > s) &= 1 - (1-p_{n,i})\sum\limits_{\ell=1}^s (p_{n,i})^{\ell-1}\\
\label{eq:appB_3_prev}
    & = 1 - (1-(p_{n,i})^s)\\
    %&=p^s\\
    \label{eq:appB_3}
    &\geq (1-\varepsilon_{n,i})^s\ ,
\end{align}
where we have defined
\begin{align}
    \varepsilon_{n,i}\triangleq \frac{1}{2}\exp\bigg (-\sqrt{\frac{n}{K}}\frac{(\mu_{a^\star}-\mu_i-\Delta_{\min}/2)^2}{4\sigma^2}\bigg )\ ,
\end{align}
and~(\ref{eq:appB_3}) is a result of~(\ref{eq:prob_bound}). Finally, we have
\begin{align}
   \label{eq:appB_4} \E_n[T_i^n] &= \sum\limits_{s=0}^\infty \P(T_i^n>s) \stackrel{(\ref{eq:appB_3})}{\geq} \sum\limits_{s=0}^\infty (1-\varepsilon_{n,i})^s =\frac{1}{\varepsilon_{n,i}}\ .
\end{align}
Similarly, an upper bound on $\E_{\bmu}[T_i^n]$ can be obtained by following the same line of arguments as~(\ref{eq:appB_2_prev0})-(\ref{eq:appB_3_prev}), and subsequently, upper-bounding $\P(T_i^n>s)$ by $(1-\varepsilon_{n,i}^\prime)^s$, where 
%replacing $\varepsilon_n$ by $\varepsilon_n^\prime$ obtained in (\ref{eq:appB_8}), where 
\begin{align}
    \varepsilon_{n,i}^\prime&\;\;\triangleq \frac{1}{\sqrt{2\pi e}}\exp\Bigg (- \sqrt{n}
    \left (\mu_{a^\star}-\mu_i+\frac{\Delta_{\min}}{2}\right)\bigg(\sqrt{n}\frac{\mu_{a^\star}-\mu_i+\Delta_{\min}/2}{4\sigma^2}+\frac{1}{\sqrt{2}\sigma}\bigg) \Bigg )\ ,
\end{align}
which follows from~(\ref{eq:appB_8}). This yields
\begin{align}
\label{eq:last_1}
    \E_n[T^n_i]&\leq \sqrt{2\pi\e}\exp\bigg ( \sqrt{n}\left ( \mu_{a^\star} - \mu_i + \frac{\Delta_{\min}}{2}\right )\left ( \sqrt{n}\frac{\mu_{a^\star}-\mu+\Delta_{\min}/2}{4\sigma^2}+\frac{1}{\sqrt{2}\sigma}\right )\bigg )\ .
\end{align}
Note that for all $n>32/(9\Delta_{\min}^2)$, we have
\begin{align}
\label{eq:last_2}
    \frac{n(\mu_{a^\star}-\mu_i + \Delta_{\min}/2)^2}{4\sigma^2}\; > \; \frac{\sqrt{n}(\mu_{a^\star}-\mu_i + \Delta_{\min}/2)}{\sqrt{2}\sigma}\ .
\end{align}
Combining~(\ref{eq:last_1}) and~(\ref{eq:last_2}), we obtain the upper-bound in Theorem~\ref{theorem:TTTS_samples}. Next, we prove that there exists $\delta(N_0)>0$ such that for any $\delta\in(0,\delta(N_0))$, the stopping time of the TTTS algorithm, denoted by $\tau_{\sf TTTS}$, almost surely satisfies $\tau_{\sf TTTS}> N_0$. Note that $\tau_{\sf TTTS}$ is defined as
\begin{align}
\label{eq:stop_TTTS}
    \tau_{\sf TTTS}\triangleq \inf\left\{n\in\N : \Lambda_n(b^n_1,b^n_2) > c_{n,\delta}\right\}\ ,
\end{align}
where $c_{n,\delta}$ has been defined in~(\ref{eq:stop_G}). Let us set 
\begin{align}
    \delta(N_0)\triangleq (K-1)\exp\left ( -\left (\frac{N_0(\Delta_{\max}+ \Delta_{\min}/2)}{8\sigma^2} - 2\log (4 + \log N_0)\right )\right )\ .
\end{align}
For any $\delta\in(0,\delta(N_0))$, the GLLR almost surely satisfies
\begin{align}
    \forall n \in\{1,\cdots,N_0\}\ , \quad \Lambda_n(b^n_1,b^n_2) < c_{n,\delta}\ .
    \label{eq:lasteq_1}
\end{align}
Combining~(\ref{eq:stop_TTTS}) and~(\ref{eq:lasteq_1}), for any $\delta\in(0,\delta(N_0))$, we almost surely have $\tau_{\sf TTTS}>N_0$.

\bibliographystyle{IEEEbib}
\bibliography{BAIRef,BanditRef}
\end{document}